\documentclass{article}

\usepackage[T1]{fontenc}
\usepackage[utf8]{inputenc}
\usepackage{xcolor}
\usepackage{colortbl}
\definecolor{best}{RGB}{214,231,250}
\definecolor{second}{RGB}{220,242,220}
\definecolor{overbudget}{RGB}{178,34,34}
\definecolor{infeasiblebg}{RGB}{250,224,224}
\definecolor{bestfont}{RGB}{0,70,168}
\definecolor{secondfont}{RGB}{56,166,0}
\usepackage{iclr2027_conference,times}
\iclrfinalcopy
\usepackage{amsmath,amssymb,amsfonts,amsthm,mathtools,bm}
\usepackage{microtype}
\usepackage{graphicx}
\usepackage{float}
\usepackage{flafter}
\usepackage{placeins}
\usepackage{booktabs}
\usepackage{tabularx}
\usepackage{array}
\usepackage{multirow}
\usepackage{enumitem}
\usepackage{algorithm}
\usepackage{algpseudocode}
\usepackage{hyperref}
\usepackage{url}
\usepackage[nameinlink,capitalise,noabbrev]{cleveref}
\hypersetup{colorlinks=true,citecolor=blue!55!black,linkcolor=blue!55!black,urlcolor=blue!55!black}
\newtheorem{theorem}{Theorem}

\newtheorem{corollary}[theorem]{Corollary}

\theoremstyle{remark}

\theoremstyle{definition}
\newtheorem{problem}{Problem}
\crefname{theorem}{Theorem}{Theorems}
\crefname{proposition}{Proposition}{Propositions}
\crefname{corollary}{Corollary}{Corollaries}
\crefname{lemma}{Lemma}{Lemmas}
\crefname{remark}{Remark}{Remarks}
\crefname{problem}{Problem}{Problems}
\crefname{algorithm}{Algorithm}{Algorithms}

\newcommand{\KL}{D_{\mathrm{KL}}}
\newcommand{\E}{\mathbb E}

\newcommand{\N}{\mathcal N}

\newcommand{\stopgrad}{\operatorname{stopgrad}}

\title{Constrained Flow Policy Updates: \\ A Generalized Schr\"odinger Bridge View}
\author{Boyang Li, Matthew Kim, Sylvia Herbert \\
University of California San Diego}

\begin{document}
\maketitle
\lhead{Preprint}

\begin{abstract}
Online safe reinforcement learning (RL) seeks policies that maximize reward
while satisfying safety constraints.  Reward and safety can induce multimodal
action distributions, challenging the prevailing primal--dual methods:
Gaussian actors may collapse onto a single suboptimal mode, and optimization
over the nonconvex Lagrangian landscape can be unstable.  Diffusion and flow
policies can represent such distributions, but recent work with a diffusion
actor relies on estimating and matching the score of an augmented-Lagrangian
target policy.  Instead, we differentiate the augmented objective directly
through the generation path of a flow policy, so no score needs to be
estimated.  Because a flow policy lacks a readily available action
log-density for entropy regularization, we build on the density-free
kinetic-energy regularizer of FLAC, a recent reward-only method, and propose
\textbf{R}eparameterized \textbf{A}ugmented-Lagrangian \textbf{F}low
\textbf{A}ctor with \textbf{L}east \textbf{E}nergy (RAFALE), an off-policy
actor--critic method for safe RL.  We formulate its update as a constrained
one-ended generalized Schr\"odinger bridge and show that, for each source draw,
this path-space problem is exactly an entropy-regularized problem in action
space.  At positive noise, its solution reweights the reward-only action
distribution only where the estimated cost exceeds a threshold set by the
Lagrange multiplier.  As the noise vanishes, the optimal value converges to
that of a least-energy map objective that the flow policy optimizes directly.
Across seven Safety-Gymnasium tasks, RAFALE achieves competitive reward with mean
final cost within budget on every task, whereas strong baselines trade one
for the other; ablations support the necessity of both its augmented objective
and its flow actor.
\end{abstract}

\section{Introduction}
\label{sec:intro}

Safe reinforcement learning (RL) seeks policies that maximize reward while
respecting a cost budget.  The constrained Markov decision process (CMDP)
expresses this requirement as a bound on expected cumulative cost
\citep{altman1999cmdp,achiam2017cpo}.  Primal--dual methods address it by
alternating policy updates with adjustments to a Lagrange multiplier
\citep{stooke2020pid}, usually with a Gaussian actor.  Yet reward and cost can
favor separated regions of the action space, producing target distributions
that a single Gaussian cannot represent.  Moreover, the nonconvex Lagrangian
landscape can make primal--dual training unstable \citep{cheng2026algd}.
Augmented Lagrangian methods address this latter difficulty by modifying the
constraint term \citep{rockafellar1973multiplier,bertsekas1982constrained,wu2024cal}.

Diffusion and flow policies are expressive enough to represent such
distributions in online RL.  The methods that train them differ in how the
critic enters generation \citep{gao2026flowrltaxonomy}: 1) Q-value guidance
matches the score network to the critic's action gradient
\citep{psenka2024qsm}; 2) weighted matching regresses the model onto a score
estimated from critic-weighted samples
\citep{ma2025rsm,li2026rfm,kim2026flag,ma2026gempo}; 3) policy-gradient methods
optimize an on-policy surrogate \citep{mcallister2025fpo,yi2026fpocontrol}; and
4) reparameterization differentiates the critic through the sampler
\citep{wang2024dacer,celik2025dime,lv2026flac}.  Related safe RL methods with
generative policies use the first two approaches: score-based updates.
Offline, FISOR extracts a diffusion policy by feasibility-guided weighted
regression \citep{zheng2024fisor}.  Online, ALGD trains a diffusion policy to
match an estimated score of the Boltzmann policy whose energy is the augmented
Lagrangian \citep{cheng2026algd}.

A direct update would avoid estimating a score altogether.
Reparameterization provides one for Gaussian actors, which are updated from the
pathwise gradient of the critic at the generated action.  For diffusion and flow
policies, however, such an update is harder to regularize.  Exploration requires
the policy to remain stochastic \citep{ziebart2010modeling,haarnoja2018sac}, but
the usual entropy term needs an action log-density that these samplers do not
readily provide.  Reward-only methods therefore estimate the entropy along the
sampling chain \citep{celik2025dime} or regularize a tractable proxy
\citep{wang2024dacer}.  FLAC offers a promising alternative for reward
maximization: it regularizes the generation process in path space, keeping it
close to a high-entropy reference process through a one-ended generalized
Schr\"odinger bridge \citep{lv2026flac}.  This yields a density-free
kinetic-energy regularizer.  For safe RL, three questions remain: what this
path-space regularization does to the action distribution, how the constraint
reshapes that distribution, and how it relates to the deterministic flow update
used in training.

Building on FLAC, we propose \textbf{R}eparameterized \textbf{A}ugmented-Lagrangian \textbf{F}low
\textbf{A}ctor with \textbf{L}east \textbf{E}nergy (RAFALE).  It evaluates the
augmented Lagrangian at the generated action and differentiates it through the
flow, so the constraint reaches the policy through a pathwise gradient rather
than an estimated score.  We answer the three questions through a constrained
one-ended Schr\"odinger bridge.  Our central
observation is that, for each source draw, this path-space problem is exactly
an entropy-regularized problem in action space.  Keeping the generation process
close to Brownian motion balances the augmented Lagrangian plus a displacement
from the source draw against the entropy of the action distribution, so the
solution becomes a Gibbs distribution.  The constraint reweights this distribution: actions whose estimated cost exceeds a threshold
set by the multiplier are downweighted increasingly with that cost, while the
relative probabilities of the remaining actions are unchanged.  The vanishing-noise limit of this problem
gives a deterministic-map objective whose restriction to a midpoint sampler is
exactly the actor loss of RAFALE.

\begin{table}[t]
\centering\scriptsize
\caption{Positioning by setting, actor class, constraint type, and
actor-update route; a related taxonomy appears in \citet{gao2026flowrltaxonomy}.
RAFALE occupies the online, constrained, generative-policy, reparameterization
cell.}
\label{tab:positioning}
\setlength{\tabcolsep}{4pt}
\begin{tabular}{@{}p{5.95cm}llll@{}}
\toprule
Method & Setting & Actor & Constraint & Actor update \\
\midrule
PID-Lag \citep{stooke2020pid} & online & Gaussian & expected cost & policy gradient \\
CAL \citep{wu2024cal} & online & Gaussian & expected cost & reparameterization \\
RCRL \citep{yu2022rcrl} & online & Gaussian & state-wise & reparameterization \\
RESPO \citep{ganai2023respo} & online & Gaussian & state-wise & policy gradient \\
\midrule
QSM \citep{psenka2024qsm} & online & diffusion & none & Q-value guidance \\
RSM, RFM, FLAG \citep{ma2025rsm,li2026rfm,kim2026flag} & online & diffusion, flow & none & weighted matching \\
FPO \citep{mcallister2025fpo} & online & flow & none & policy gradient \\
DACER, DIME \citep{wang2024dacer,celik2025dime} & online & diffusion & none & reparameterization \\
FLAC \citep{lv2026flac} & online & flow & none & reparameterization \\
\midrule
FISOR \citep{zheng2024fisor} & offline & diffusion & state-wise & weighted matching \\
ALGD \citep{cheng2026algd} & online & diffusion & expected cost & Q-value guidance \\
RAFALE (ours) & online & flow & expected cost & reparameterization \\
\bottomrule
\end{tabular}
\end{table}

Our contributions are threefold:
\begin{itemize}[leftmargin=1.4em,itemsep=1pt,topsep=2pt,parsep=0pt]
\item \textbf{A constrained Schr\"odinger-bridge analysis of flow updates.}  We show that,
  for each source draw, the constrained one-ended Schr\"odinger bridge is an
  entropy-regularized problem in action space.  We then characterize its Gibbs
  solution and safety tilt and derive its vanishing-noise map objective
  (\cref{sec:bridge}).
\item \textbf{A direct flow-policy algorithm.}  RAFALE optimizes the map
  objective through a midpoint sampler, combining the constraint term with
  kinetic regularization and off-policy critic learning (\cref{sec:method}).
\item \textbf{Evaluation in online safe RL.}  On seven Safety-Gymnasium tasks,
  RAFALE achieves competitive reward with mean final cost within budget on
  every task.  Ablations compare the augmented objective with standard and
  hinge alternatives, and the flow actor with Gaussian and mixture actors (\cref{sec:experiments}).
\end{itemize}

\section{Related Work}
\label{sec:related}

Due to space limits, we discuss here only the work closest to RAFALE;
\cref{app:related} describes the four routes of \cref{sec:intro} and safe RL
with generative actors in detail.  RAFALE follows the fourth route, and
\cref{tab:positioning} places representative methods of each route by setting,
actor class, and constraint type.

Within safe RL, CAL applies an augmented Lagrangian to a policy-level expected
constraint residual with a Gaussian actor \citep{wu2024cal}.  ALGD applies it at
the action level and, following the first route, estimates the noisy-time score of the
resulting Boltzmann policy for a diffusion actor \citep{cheng2026algd}.  We
retain the action-level objective but optimize it through a flow sampler, along the fourth route.  This expected-cost formulation differs from the state-wise
reachability constraints used by RCRL and RESPO \citep{yu2022rcrl,ganai2023respo}
and by the offline generative methods FISOR and EpiFlow
\citep{zheng2024fisor,tayal2026epiflow}.

KL control and Schr\"odinger bridges provide the basis for our
reference-process regularizer
\citep{todorov2006lsmdp,leonard2014survey,chen2016sb,zhang2022pis,liu2024gsbm}.
FLAC derives kinetic-energy regularization from this perspective for
reward-only flow policies \citep{lv2026flac}, and SoftGAC decomposes a
path-space regularizer at the endpoint for a reward-only bridge actor
\citep{he2026softgac}.  We extend this analysis to the augmented Lagrangian: with
a Brownian reference, the endpoint term becomes a displacement plus a
conditional entropy, and the noise-free limit gives the objective of a
deterministic flow actor.

\section{The Constrained Flow-Update Problem}
\label{sec:background}

\subsection{From an Episodic Constraint to an Actor Objective}

We consider a CMDP \((\mathcal S,\mathcal A,p,d_0,r,c,\gamma,H)\), where
\(\mathcal S\) and \(\mathcal A\) are continuous state and action spaces, \(p\)
is the transition kernel, and \(d_0\) is the initial-state distribution.  The
reward \(r\), nonnegative cost \(c\), discount \(\gamma\in(0,1)\), and episode
horizon \(H\) define the objective
\begin{equation}
\max_\pi\; J_r(\pi):=\E_\pi\Big[\sum_{k<H}\gamma^k r(s_k,a_k)\Big]
\quad\text{s.t.}\quad
J_c(\pi):=\E_\pi\Big[\sum_{k<H}c(s_k,a_k)\Big]\le b ,
\label{eq:cmdp}
\end{equation}
where \(s_0\sim d_0\), \(a_k\sim\pi(\cdot\mid s_k)\), and
\(s_{k+1}\sim p(\cdot\mid s_k,a_k)\).  The budget \(b\) thus limits the expected
undiscounted cost of an episode rather than the cost at every state.

Although \cref{eq:cmdp} evaluates complete episodes, the actor must compare
individual actions.  Reward and cost action values make this comparison
possible by accounting for what follows each action under the current policy:
\begin{equation}
Q_x(s,a):=\E_\pi\Big[\sum_{k\ge0}\gamma^k x(s_k,a_k)\,\Big|\,s_0=s,\ a_0=a\Big],
\qquad x\in\{r,c\}.
\label{eq:critics}
\end{equation}
We estimate these values with off-policy critics trained on replay data
\citep{liu2022cvpo,wu2024cal}.  Since \(Q_c\) is discounted while the episode
budget is not, we use a nominal budget \(h\) in critic units, converted as
specified in \cref{app:algorithm}.  Following primal--dual methods
\citep{altman1999cmdp,achiam2017cpo,stooke2020pid}, we combine the two critics
with a nonnegative multiplier \(\lambda\).  For given critics and multiplier,
the standard action-level Lagrangian is
\begin{equation}
L(s,a):=-Q_r(s,a)+\lambda\,\big(Q_c(s,a)-h\big).
\label{eq:linear-step}
\end{equation}
Minimizing its expectation favors higher reward while accounting for future
cost, with \(\lambda\) setting the relative weight of the two objectives.

Primal--dual updates built on \(L\), however, are prone to oscillation: the
multiplier integrates past constraint violations, which makes the cost
oscillate around the budget \citep{stooke2020pid}.  The augmented Lagrangian is
the classical remedy, because it smooths the dual problem
\citep{rockafellar1973multiplier,rockafellar1976augmented,bertsekas1982constrained},
and \citet{cheng2026algd} analyze its effect on the action-level landscape of a
generative actor.  At
the action level, it replaces the linear constraint term of \(L\) by a
piecewise quadratic one:
\begin{equation}
L_A(s,a):=-Q_r(s,a)+\Phi_{\lambda,\rho}\big(Q_c(s,a)-h\big),
\qquad
\Phi_{\lambda,\rho}(y):=\frac{[\lambda+\rho y]_+^2-\lambda^2}{2\rho},
\label{eq:aug}
\end{equation}
where \([y]_+=\max\{y,0\}\) and \(\rho>0\) sets the strength of the quadratic
part.  To see which actions this term affects, we rewrite it around the
activation threshold \(\tau\):
\begin{equation}
L_A(s,a)=-Q_r(s,a)+\frac{\rho}{2}\big[Q_c(s,a)-\tau\big]_+^2-\frac{\lambda^2}{2\rho},
\qquad \tau:=h-\lambda/\rho .
\label{eq:threshold}
\end{equation}
Below \(\tau\), \(L_A\) differs from \(-Q_r\) only by an action-independent
constant.  Above \(\tau\), the added term grows quadratically with the excess of
the estimated cost over \(\tau\), which is the active-set structure of the
augmented Lagrangian \citep{hintermuller2006pathfollowing}.  The multiplier
determines where this influence begins: \(\tau\) lies below the nominal budget
\(h\) whenever \(\lambda>0\), and a larger \(\lambda\) lowers it further,
enlarging the set of actions on which the constraint term acts.

\subsection{A Direct Actor Update for a Flow Policy}
\label{sec:flow}

The objective \(L_A\) tells us how to compare actions, but learning a policy
also requires a way to transmit this objective to the actor.  A flow policy
provides a differentiable generation process: a source sample
\(x_0\sim\mu_0\) evolves under a state-conditioned velocity field,
\begin{equation}
\frac{dX_t}{dt}=v(s,X_t,t),\qquad X_0=x_0,\qquad t\in[0,1].
\label{eq:flow-policy}
\end{equation}
At a fixed state \(s\), a numerical solver produces the endpoint \(a:=X_1\)
\citep{lipman2023flow}.\footnote{For bounded environment actions, \(X_1\) is the
unsquashed output, and critic evaluations include the final tanh map.  Kinetic
energy is measured before squashing; see \cref{app:algorithm}.}  Each source
draw gives a deterministic endpoint, while sampling \(x_0\) induces the policy
distribution.  Because the solver is differentiable, the critic can update this
generator through the action it produces.  We therefore evaluate \(L_A\) at the
endpoint and optimize its expectation, extending the reparameterized updates of
reward-only generative policies \citep{wang2024dacer,celik2025dime,lv2026flac}:
\begin{problem}[Direct constrained flow update]
\label{prob:direct}
\begin{equation}
\min_v\;\E_{s\sim\mathcal D,\,x_0\sim\mu_0}\big[L_A\big(s,X_1\big)\big],
\label{eq:naive-step}
\end{equation}
\end{problem}
where \(s\sim\mathcal D\) denotes sampling states from the replay buffer.  The
constraint thus reaches the actor through the same pathwise derivative as the
reward term, without an intermediate score estimate.  \Cref{prob:direct},
however, contains no entropy regularization.  Its minimizers do not depend on
\(x_0\), so the update can map every source draw to the same minimizer of
\(L_A\), and the policy may collapse prematurely and stop exploring.

\section{A Safety-Tilted Schr\"odinger Bridge}
\label{sec:bridge}

Maximum-entropy RL avoids this collapse through entropy regularization, which
requires an action log-density that our flow sampler does not readily provide
\citep{haarnoja2018sac}.  Following
FLAC \citep{lv2026flac}, we therefore regularize the generation process in path
space, where the regularizer needs no action log-density, instead of the action
distribution in action space.  The central result of this section is that, for
each source draw, this path-space regularization is exactly an
entropy-regularized problem in action space (\cref{thm:action-form}); its solution and noise-free limit
then determine the update of RAFALE.

\subsection{From a Reference Process to Kinetic Regularization}
\label{sec:kinetic}

We take as the reference a zero-drift Brownian process, which adds isotropic
noise to each source draw:
\begin{equation}
dX_t=\sqrt\varepsilon\,dW_t,\qquad X_0\sim\mu_0,\qquad t\in[0,1],
\label{eq:reference}
\end{equation}
where \(W\) is a standard Brownian motion and \(\varepsilon>0\).  We write
\(R^\varepsilon\) for its path distribution and \(R_1^\varepsilon\) for the
distribution of its endpoint \(X_1\), which is the source distribution spread
by Gaussian noise of variance \(\varepsilon\).  The desired generation process
starts from \(\mu_0\), ends at actions with low \(L_A\), and stays close to
\(R^\varepsilon\).  The following problem formalizes this trade-off in path space:
\begin{problem}[Constrained one-ended Schr\"odinger bridge]
\label{prob:bridge}
\begin{equation}
\mathcal B^\varepsilon(s)
:=\inf_{P:\,P_0=\mu_0}
\Big\{\E_P\big[L_A(s,X_1)\big]+\alpha\varepsilon\,\KL(P\Vert R^\varepsilon)\Big\},
\qquad \alpha>0 .
\label{eq:bridge-objective}
\end{equation}
\end{problem}
\Cref{prob:bridge} is a one-ended generalized Schr\"odinger bridge
\citep{todorov2006lsmdp,zhang2022pis,liu2024gsbm}.  Unlike a classical
Schr\"odinger bridge, it fixes only the source distribution \(P_0=\mu_0\) and
leaves the action distribution \(P_1\) free, since we have no target action
distribution to match.  Following FLAC \citep{lv2026flac}, we call \(L_A\), which
shapes \(P_1\), the terminal potential.

\Cref{prob:bridge}, however, cannot be optimized with our sampler as written.
Its KL term is defined by the density ratio \(dP/dR^\varepsilon\) between path
distributions, whereas the sampler only generates paths from a velocity field.
We therefore express the KL through that field.  Let the generation process follow a
velocity field \(v\) under the reference noise,
\begin{equation}
dX_t=v(s,X_t,t)\,dt+\sqrt\varepsilon\,dW_t,\qquad X_0\sim\mu_0,\qquad t\in[0,1],
\label{eq:controlled}
\end{equation}
and let \(P\) denote its path distribution.  Because \cref{eq:controlled} and
\cref{eq:reference} share the noise, Girsanov's theorem \citep{leonard2014survey}
gives
\begin{equation}
\varepsilon\,\KL(P\Vert R^\varepsilon)
=\frac12\,\E_P\Big[\int_0^1\|v(s,X_t,t)\|^2\,dt\Big]
\label{eq:girsanov}
\end{equation}
under the conditions of \cref{app:proof-bridge}.  Scaled by \(\varepsilon\), the
divergence from the reference is thus the expected kinetic energy of the drift.
Substituting it into \cref{prob:bridge} gives an equivalent reformulation over
velocity fields, with the same value and the same optimal path distribution
(\cref{app:proof-bridge}):
\begin{problem}[Equivalent kinetic-energy reformulation]
\label{prob:control}
\begin{equation}
\mathcal B^\varepsilon(s)
=\inf_v\;\E\Big[L_A(s,X_1)+\frac{\alpha}{2}\int_0^1\|v(s,X_t,t)\|^2\,dt\Big],
\qquad X \text{ follows \cref{eq:controlled}}.
\label{eq:bridge-control}
\end{equation}
\end{problem}
Both terms of its objective can be estimated from generated paths: the
terminal potential from the critics at \(X_1\), and the kinetic energy from the
velocities along the path.  Unlike \cref{prob:bridge}, \cref{prob:control} thus requires no density ratio;
\cref{sec:update} connects it to the deterministic sampler of RAFALE.

The kinetic energy also controls the action distribution.  Since the endpoint
\(X_1\) is a function of the path, the data-processing inequality bounds the
divergence between endpoint distributions by the divergence between path
distributions (\cref{app:proof-bridge}), and \cref{eq:girsanov} turns the latter
into kinetic energy:
\begin{equation}
\KL(P_1\Vert R_1^\varepsilon)\;\le\;\KL(P\Vert R^\varepsilon)
=\frac{1}{2\varepsilon}\,\E_P\Big[\int_0^1\|v(s,X_t,t)\|^2\,dt\Big].
\label{eq:endpoint-main}
\end{equation}
Minimizing the kinetic energy thus keeps the action distribution \(P_1\) close
to the broad reference distribution \(R_1^\varepsilon\) without evaluating its
density.

\subsection{Action-Space Form and Safety Tilt}
\label{sec:tilt}

The bound of \cref{eq:endpoint-main} limits the divergence of the action
distribution, but it does not reveal which action distribution
\cref{prob:control} selects.  The next result answers this exactly: for each
source draw, the path-space problem is an entropy-regularized problem in action
space, which yields both the solution and its noise-free limit.

\begin{theorem}[Action-space form]
\label{thm:action-form}
Let \(\varepsilon>0\).  Under the conditions of \cref{app:proof-bridge},
\cref{prob:bridge,prob:control} are equivalent to the following problem over
conditional action distributions:
\begin{equation}
\mathcal B^\varepsilon(s)
=\E_{x_0\sim\mu_0}\Big[\min_q\Big\{\E_{a\sim q}\big[L_A(s,a)+\alpha\,\mathcal E(x_0,a)\big]
-\alpha\varepsilon\,H(q)\Big\}\Big]+\frac{\alpha\varepsilon d_a}{2}\log(2\pi\varepsilon).
\label{eq:free-energy}
\end{equation}
Here \(q\) ranges over densities of actions, \(\mathcal E(x_0,a):=\tfrac12\|a-x_0\|^2\),
\(H(q):=-\int q(a)\log q(a)\,da\), and \(d_a\) is the action dimension.  An
optimal generation process draws its action from the minimizer in
\cref{eq:free-energy} and connects each source draw to its action by a
Brownian bridge.
\end{theorem}

We prove \cref{thm:action-form} in \cref{app:proof-bridge} from the chain rule
for the KL divergence and the Gaussian endpoint \(\N(x_0,\varepsilon I)\) of the
reference.  It is the static form of the Schr\"odinger problem
\citep{leonard2014survey}, written for our one-ended setting.  The theorem
connects the path-space view of regularization, taken by FLAC, with the
action-space view of maximum-entropy RL.  For each source draw, the KL to the
reference acts on the actions through two terms.  The displacement
\(\mathcal E(x_0,a)\) is the kinetic energy of the straight path from \(x_0\) to
\(a\) at constant speed.  The sum \(L_A+\alpha\mathcal E\) is therefore the
total energy of reaching \(a\): the terminal potential plus \(\alpha\) times
this kinetic energy.  The entropy \(H(q)\), conditional on the source draw and weighted by the temperature
\(\alpha\varepsilon\), spreads the actions.  Keeping the generation process
close to Brownian motion is therefore exactly entropy regularization of the
actions generated from each source draw, anchored to that draw by the
displacement.  The constant in \cref{eq:free-energy} does not depend on \(q\).

Because \cref{eq:free-energy} is entropy-regularized, its minimizer is a
Boltzmann distribution.  \Cref{thm:bridge-laws} gives this minimizer and
compares it with the reward-only Schr\"odinger bridge of FLAC, obtained by
replacing \(L_A\) with \(-Q_r\), to isolate the effect of the constraint.

\begin{theorem}[Gibbs solution and safety tilt]
\label{thm:bridge-laws}
Let \(\varepsilon>0\).  Let \(q_A^\varepsilon(\cdot\mid s,x_0)\) denote the
minimizer in \cref{eq:free-energy}, and let \(q_R^\varepsilon(\cdot\mid s,x_0)\)
denote the minimizer when \(L_A\) is replaced by \(-Q_r\).  Under the conditions of \cref{app:proof-bridge}:

\emph{(a) Gibbs solution.}
\begin{align}
q_R^\varepsilon(a\mid s,x_0)
&\propto\exp\Big\{-\frac{-Q_r(s,a)+\alpha\,\mathcal E(x_0,a)}{\alpha\varepsilon}\Big\},
\nonumber\\
q_A^\varepsilon(a\mid s,x_0)
&\propto\exp\Big\{-\frac{L_A(s,a)+\alpha\,\mathcal E(x_0,a)}{\alpha\varepsilon}\Big\}.
\label{eq:gibbs}
\end{align}

\emph{(b) Safety tilt.}  The two densities satisfy
\begin{align}
q_A^\varepsilon(a\mid s,x_0)
&=q_R^\varepsilon(a\mid s,x_0)\,\frac{w^\varepsilon(s,a)}{Z^\varepsilon(s,x_0)},
\qquad
w^\varepsilon(s,a):=\exp\!\Big[-\frac{\rho}{2\alpha\varepsilon}\big[Q_c(s,a)-\tau\big]_+^2\Big],
\label{eq:safety-tilt}\\
Z^\varepsilon(s,x_0)&:=\E_{a'\sim q_R^\varepsilon(\cdot\mid s,x_0)}\big[w^\varepsilon(s,a')\big]\le1 .
\nonumber
\end{align}
\end{theorem}

Part (a) follows from \cref{eq:free-energy} by the Gibbs variational principle
(\cref{app:proof-bridge}).  The solution compares two actions \(a\) and \(a'\)
generated from the same source draw through their total energy alone:
\begin{equation}
\frac{q_A^\varepsilon(a\mid s,x_0)}{q_A^\varepsilon(a'\mid s,x_0)}
=\exp\Big\{-\frac{\big[L_A(s,a)-L_A(s,a')\big]
+\alpha\big[\mathcal E(x_0,a)-\mathcal E(x_0,a')\big]}{\alpha\varepsilon}\Big\}.
\label{eq:gibbs-ratio}
\end{equation}
An action is thus more likely when it lowers \(L_A\) by more than \(\alpha\)
times its extra displacement.

Part (b) isolates the effect of the constraint.  The normalizer
\(Z^\varepsilon(s,x_0)\) depends only on the source draw, so it cancels when two
actions from the same source draw are compared, and \cref{eq:safety-tilt} gives
\begin{equation}
\frac{q_A^\varepsilon(a\mid s,x_0)}{q_A^\varepsilon(a'\mid s,x_0)}
=\frac{q_R^\varepsilon(a\mid s,x_0)}{q_R^\varepsilon(a'\mid s,x_0)}\cdot
\frac{w^\varepsilon(s,a)}{w^\varepsilon(s,a')}.
\label{eq:tilt-ratio}
\end{equation}
If both actions satisfy \(Q_c\le\tau\), both weights equal one and the
reward-only ratio is unchanged.  Otherwise, the action with the higher estimated
cost above \(\tau\) becomes less likely to be selected relative to the other,
increasingly with that cost.  The constraint therefore lowers the
probability of selecting actions with high estimated cost
(\cref{app:tilt-consequences}).

\subsection{The Vanishing-Noise Limit}
\label{sec:update}

\Cref{thm:action-form,thm:bridge-laws} characterize the Schr\"odinger bridge at positive noise,
whereas the actor of RAFALE is a deterministic flow, which sends each source
draw to a single action.  We therefore study the limit of
\cref{prob:control}, equivalently of \cref{eq:free-energy}, as
\(\varepsilon\to0\) with \(\alpha\) fixed.  In \cref{eq:free-energy}, this limit
removes only the entropy term and the constant, so each conditional solution concentrates
on the minimizers of the total energy \(L_A+\alpha\mathcal E\).  \Cref{thm:vanishing}
makes this precise.

\begin{theorem}[Vanishing noise]
\label{thm:vanishing}
Under the conditions of \cref{app:proof-bridge}, for each state \(s\),
\begin{equation}
\lim_{\varepsilon\downarrow0}\mathcal B^\varepsilon(s)
=\inf_{T}\;\E_{x_0\sim\mu_0}\Big[L_A\big(s,T(x_0)\big)
+\alpha\,\mathcal E\big(x_0,T(x_0)\big)\Big],
\label{eq:vanishing-noise-map}
\end{equation}
where \(T\) ranges over the square-integrable maps
\(L^2(\mu_0;\mathbb R^{d_a})\) from source draws to actions and \(d_a\) is the
action dimension.  The infimum is attained, and \(T\) attains it if and only
if, for \(\mu_0\)-almost every \(x_0\),
\begin{equation}
T(x_0)\in\operatorname*{arg\,min}_a\,\big\{L_A(s,a)+\alpha\,\mathcal E(x_0,a)\big\}.
\label{eq:prox-map}
\end{equation}
\end{theorem}

Without noise, \cref{prob:control} thus reduces to choosing one action per
source draw: its value converges to that of \cref{eq:vanishing-noise-map}, and
an optimal map sends each source draw to a minimizer of its total energy.

The limit does not remove the randomness of the policy.  The positive-noise
solution of \cref{eq:free-energy} is random for two reasons: the source draw
\(x_0\) is random, and, given \(x_0\), the entropy term spreads the actions.
The limit problem of \cref{thm:vanishing} removes only the second.  The first remains because the displacement
makes the minimizer in \cref{eq:prox-map} depend on \(x_0\): for differentiable
\(L_A\), the optimal map sends distinct source draws to distinct actions
(\cref{app:injective,app:claims}).  The minimizers of \cref{prob:direct}, by
contrast, do not depend on \(x_0\), so \cref{prob:direct} can send every source
draw to the same action (\cref{sec:flow}).  In the noise-free limit, the kinetic
regularization thus acts through the source draws, geometrically, rather than
through the entropy of each conditional distribution.

\FloatBarrier

\section{Practical Online Algorithm}
\label{sec:method}
\label{sec:practical}

RAFALE minimizes the objective of \cref{thm:vanishing}, the noise-free limit of
the action-space problem \cref{eq:free-energy}, over maps given by a
velocity network \(v_\theta\) and one explicit midpoint step.  The step, which FLAC also uses, requires two network evaluations:
\begin{equation}
X_1=x_0+\delta_\theta(s,x_0),\qquad
\delta_\theta(s,x_0)=v_\theta\!\left(s,x_0+\tfrac12v_\theta(s,x_0,0),\tfrac12\right).
\label{eq:midpoint-main}
\end{equation}
The same evaluations give the kinetic energy of the step,
\begin{equation}
K_\theta(s,x_0):=\tfrac12\|\delta_\theta(s,x_0)\|^2
=\tfrac12\|X_1-x_0\|^2=\mathcal E(x_0,X_1),
\label{eq:exact-kinetic}
\end{equation}
which is exactly the displacement \(\mathcal E(x_0,X_1)\) of
\cref{thm:vanishing}, so the one-step sampler evaluates the regularizer of the
limit objective without approximation.  Evaluating \cref{eq:vanishing-noise-map}
at the midpoint map and averaging over replay states therefore gives the actor
loss (\cref{app:solver-identity}):
\begin{equation}
\mathcal L(\theta):=\E_{s\sim\mathcal D,\,x_0\sim\mu_0}\Big[-Q_r(s,X_1)
+\Phi_{\lambda,\rho}\big(Q_c(s,X_1)-h\big)
+\alpha K_\theta(s,x_0)\Big].
\label{eq:actor-loss}
\end{equation}
We estimate the loss with one source draw per replay state and backpropagate
through both velocity evaluations, so neither the Gibbs density of
\cref{thm:bridge-laws} nor its score is estimated during training.  

To use \cref{eq:actor-loss} online, we learn its critics from replay and adjust
the multiplier using completed episodes; \cref{alg:method} in \cref{app:algorithm} summarizes the learner.  Two reward critics
use clipped double Q-learning \citep{fujimoto2018td3}, with their minimum
supplying \(Q_r\), alongside a single cost critic \(Q_c\).  Following FLAC
\citep{lv2026flac}, the reward backup subtracts \(\alpha\) times the kinetic
energy of the next action, so \(Q_r\) is a kinetic-energy-regularized reward
critic.  The multiplier instead follows a projected
proportional--integral rule driven by measured episode costs
\citep{stooke2020pid}.  Data collection and evaluation
use the same midpoint sampler, without candidate selection.

\paragraph{Kinetic-energy adaptation.}
Because kinetic energy varies across tasks and training stages, a fixed
\(\alpha\) can constrain the generator too strongly or too weakly.  We therefore
use the automatic tuning rule of FLAC, updating \(\log\alpha\) by gradient
descent on
\begin{equation}
\mathcal L_\alpha=\log\alpha\,\big(K_{\mathrm{tgt}}-\stopgrad(\widehat K)\big),
\qquad K_{\mathrm{tgt}}=1.125\,d_a,
\label{eq:log-alpha-loss}
\end{equation}
where \(\widehat K\) is the minibatch mean kinetic energy and \(d_a\) is the
action dimension.  Energy above the target increases \(\alpha\) and strengthens
the regularizer, while energy below it decreases \(\alpha\).

\begin{table}[!htbp]
\caption{Final episodic reward and cost (mean \(\pm\) SD, five seeds; \(b=10\),
or \(25\) on CarGoal); RESPO trains for 9M--10M environment steps, the others
for 1M--3M (\cref{app:experiments}).  Red: mean cost above budget.  Among
entries within budget, bold blue and bold green mark the highest and
second-highest reward.}
\label{tab:main-results}
\centering
\scriptsize
\renewcommand{\arraystretch}{0.85}
\setlength{\tabcolsep}{3pt}
\begin{tabular}{llccccc}
\toprule
Task & Metric & RESPO & CAL & RCRL & ALGD & RAFALE (ours) \\
\midrule
\multirow{2}{*}{Swimmer} & Reward & \(36\pm3\) & \cellcolor{infeasiblebg}\(41\pm3\) & \(20\pm15\) & \textcolor{secondfont}{\(\bm{52\pm13}\)} & \textcolor{bestfont}{\(\bm{93\pm23}\)} \\
 & Cost & \(7.5\pm0.6\) & \cellcolor{infeasiblebg}\(16.6\pm6.7\) & \(4.0\pm7.6\) & \(2.9\pm3.5\) & \(5.4\pm2.7\) \\
\specialrule{0.3pt}{0.6pt}{0.6pt}
\multirow{2}{*}{HalfCheetah} & Reward & \textcolor{secondfont}{\(\bm{2344\pm320}\)} & \cellcolor{infeasiblebg}\(2650\pm111\) & \cellcolor{infeasiblebg}\(1796\pm2493\) & \cellcolor{infeasiblebg}\(2817\pm69\) & \textcolor{bestfont}{\(\bm{3025\pm55}\)} \\
 & Cost & \(9.8\pm5.9\) & \cellcolor{infeasiblebg}\(15.5\pm14.7\) & \cellcolor{infeasiblebg}\(195.2\pm434.8\) & \cellcolor{infeasiblebg}\(22.4\pm44.7\) & \(3.5\pm4.7\) \\
\specialrule{0.3pt}{0.6pt}{0.6pt}
\multirow{2}{*}{Hopper} & Reward & \textcolor{secondfont}{\(\bm{1073\pm308}\)} & \cellcolor{infeasiblebg}\(1664\pm24\) & \(1037\pm575\) & \cellcolor{infeasiblebg}\(1702\pm15\) & \textcolor{bestfont}{\(\bm{1655\pm46}\)} \\
 & Cost & \(3.1\pm5.1\) & \cellcolor{infeasiblebg}\(12.2\pm6.2\) & \(0.2\pm0.3\) & \cellcolor{infeasiblebg}\(15.7\pm12.3\) & \(0.3\pm0.4\) \\
\specialrule{0.3pt}{0.6pt}{0.6pt}
\multirow{2}{*}{Walker2d} & Reward & \cellcolor{infeasiblebg}\(2925\pm66\) & \textcolor{secondfont}{\(\bm{2894\pm56}\)} & \cellcolor{infeasiblebg}\(1758\pm871\) & \cellcolor{infeasiblebg}\(2920\pm133\) & \textcolor{bestfont}{\(\bm{3177\pm26}\)} \\
 & Cost & \cellcolor{infeasiblebg}\(10.6\pm3.2\) & \(6.3\pm3.8\) & \cellcolor{infeasiblebg}\(10.1\pm16.0\) & \cellcolor{infeasiblebg}\(17.1\pm15.2\) & \(1.3\pm1.3\) \\
\specialrule{0.3pt}{0.6pt}{0.6pt}
\multirow{2}{*}{Ant} & Reward & \textcolor{secondfont}{\(\bm{1950\pm75}\)} & \cellcolor{infeasiblebg}\(2815\pm265\) & \cellcolor{infeasiblebg}\(2332\pm1292\) & \cellcolor{infeasiblebg}\(2218\pm1735\) & \textcolor{bestfont}{\(\bm{3427\pm62}\)} \\
 & Cost & \(5.6\pm1.1\) & \cellcolor{infeasiblebg}\(11.1\pm9.4\) & \cellcolor{infeasiblebg}\(76.3\pm169.7\) & \cellcolor{infeasiblebg}\(16.8\pm11.9\) & \(4.4\pm1.7\) \\
\specialrule{0.3pt}{0.6pt}{0.6pt}
\multirow{2}{*}{Humanoid} & Reward & \textcolor{secondfont}{\(\bm{5380\pm900}\)} & \(5289\pm167\) & \cellcolor{infeasiblebg}\(5640\pm705\) & \(4530\pm829\) & \textcolor{bestfont}{\(\bm{6513\pm27}\)} \\
 & Cost & \(1.5\pm1.1\) & \(0.1\pm0.2\) & \cellcolor{infeasiblebg}\(15.2\pm34.0\) & \(0.2\pm0.5\) & \(3.2\pm3.1\) \\
\specialrule{0.3pt}{0.6pt}{0.6pt}
\multirow{2}{*}{CarGoal} & Reward & \textcolor{bestfont}{\(\bm{20.0\pm1.3}\)} & \(12.1\pm1.7\) & \cellcolor{infeasiblebg}\(1.2\pm0.8\) & \(16.5\pm2.6\) & \textcolor{secondfont}{\(\bm{18.3\pm1.6}\)} \\
 & Cost & \(20.8\pm3.9\) & \(12.3\pm2.6\) & \cellcolor{infeasiblebg}\(26.1\pm11.6\) & \(22.4\pm4.8\) & \(20.7\pm3.3\) \\
\midrule
\multicolumn{2}{l}{Tasks within budget} & 6/7 & 3/7 & 2/7 & 3/7 & \textbf{7/7} \\
\bottomrule
\end{tabular}
\end{table}

\begin{figure}[t]
\centering
\includegraphics[width=0.85\textwidth]{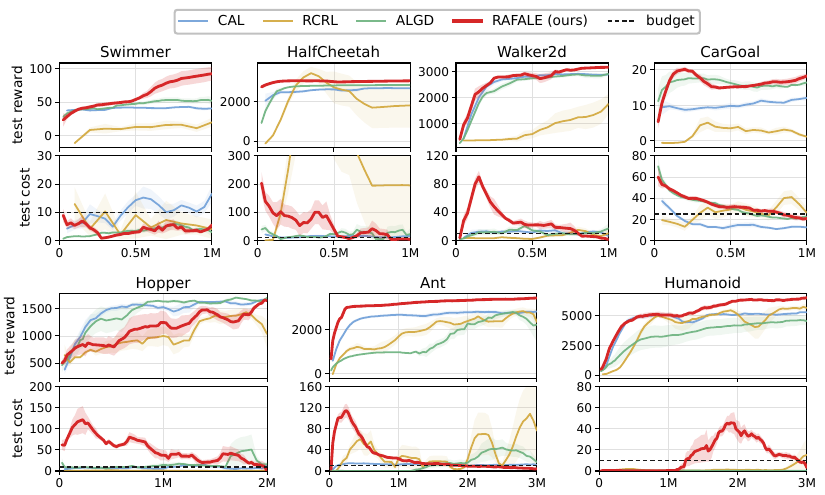}
\caption{Learning curves (mean \(\pm1\) standard error over five seeds): test
reward (upper rows) and test cost (lower rows) against environment steps.
Dashed lines mark the budget, and cost axes are truncated.  RESPO, trained for
up to 10M steps, appears in \cref{fig:curves-log}.}
\label{fig:curves-offpolicy}
\end{figure}

\section{Experiments}
\label{sec:experiments}

We test RAFALE on six velocity-constrained locomotion tasks and the navigation
task CarGoal from Safety-Gymnasium \citep{ji2023safetygymnasium}.  The episodic
cost budget is \(b=10\) for locomotion and \(b=25\) for CarGoal.  We report
final-policy reward and cost as mean \(\pm\) standard deviation over five
seeds; a run is feasible when its mean evaluation cost is at most \(b\).  All
methods are evaluated at a fixed endpoint checkpoint (\cref{app:experiments}).

\paragraph{Baselines.}
CAL \citep{wu2024cal} and ALGD \citep{cheng2026algd} are off-policy
primal--dual baselines with Gaussian and diffusion actors.  RCRL
\citep{yu2022rcrl} and RESPO \citep{ganai2023respo} instead use state-wise
reachability constraints, a stricter safety objective than an expected-cost
budget; RCRL is off-policy, and RESPO is on-policy with substantially more
environment steps.

\subsection{Main Results}
\label{sec:main-results}

RAFALE is the only method with mean final cost within budget on every task
(\cref{tab:main-results}).  Among the entries within budget, it has the highest
reward on all locomotion tasks, and RESPO has the highest reward on CarGoal.
CAL and ALGD attain competitive reward, but each exceeds the budget on more than
half of the tasks.  The reachability-based methods obtain lower reward on
locomotion, even with the longer training schedule of RESPO; RESPO is feasible
on nearly every task, whereas RCRL exceeds the budget on most of them.  \Cref{fig:curves-offpolicy} also shows transient cost excursions during
training (\cref{sec:conclusion}).

\subsection{Ablation Studies}
\label{sec:ablations}

We ablate both components on Hopper, Walker2d, and Humanoid (protocols in
\cref{app:ablation-protocol}).

\paragraph{Augmented versus standard Lagrangian.}
To separate activation at a threshold from growth above it, we compare \(L_A\)
with the standard Lagrangian \(L\) and a hinge control
\(L_C:=-Q_r+\lambda[Q_c-\tau]_+\), which keeps the threshold of \(L_A\) but
grows linearly above it.  All other learner settings, including the multiplier
update rule, are shared.  On Hopper and Walker2d, the
augmented objective gives higher final reward and lower peak and integrated
excess of evaluated cost (\cref{tab:gate-ablation-full}).  On Humanoid, final
reward and integrated excess are similar, but the augmented objective is the
only variant whose five final runs are all feasible.  The hinge variant has a
higher peak cost on every task, which suggests that the improvement depends on
how the constraint term grows above the threshold, not only on the threshold
itself.

\paragraph{Effect of the actor family.}
To compare actor families under the same constraint mechanism, we replace the
flow actor with tanh-Gaussian and eight-component tanh-GMM actors, each trained
with its native regularizer, and leave the remaining settings unchanged.  The
flow actor attains the lowest mean final cost on all three tasks with
comparable reward (\cref{tab:actor-family}), although its cost advantage over
the GMM is small on Walker2d.  Only the flow configuration has feasible final
policies in every run, so the flow actor satisfies the budget more reliably
without sacrificing reward.

\begin{table}[!htbp]
\caption{Lagrangian ablation (five seeds; mean \(\pm\) SD where shown): final
reward and cost, peak evaluated cost (Peak), integrated excess above budget
(Excess), and feasible final seeds (Feas.).  Bold blue: highest reward; bold
black: lowest Peak and Excess, most feasible seeds.}
\label{tab:gate-ablation-full}
\centering
\scriptsize
\renewcommand{\arraystretch}{0.88}
\setlength{\tabcolsep}{2.8pt}
\begin{tabular}{llrrrrc}
\toprule
Task & Variant & Reward & Cost & Peak & Excess & Feas. \\
\midrule
\multirow{3}{*}{Hopper}
& Augmented \(A\) (RAFALE) & \textcolor{bestfont}{\(\bm{1655\pm46}\)} & \(0.3\pm0.4\) & \(\bm{174.6}\) & \(\bm{92.3\pm15.9}\) & 5/5 \\
& Hinge \(C\) & \(1547\pm33\) & \(0.8\pm1.2\) & 902.9 & \(251.2\pm30.3\) & 5/5 \\
& Standard \(U\) & \(1208\pm573\) & \(0.0\pm0.1\) & 767.8 & \(240.6\pm70.2\) & 5/5 \\
\midrule
\multirow{3}{*}{Walker2d}
& Augmented \(A\) (RAFALE) & \textcolor{bestfont}{\(\bm{3177\pm26}\)} & \(1.3\pm1.3\) & \(\bm{125.9}\) & \(\bm{20.3\pm3.8}\) & 5/5 \\
& Hinge \(C\) & \(3075\pm89\) & \(0.0\pm0.0\) & 390.1 & \(55.9\pm28.4\) & 5/5 \\
& Standard \(U\) & \(2939\pm65\) & \(1.2\pm2.1\) & 443.8 & \(83.7\pm42.9\) & 5/5 \\
\midrule
\multirow{3}{*}{Humanoid}
& Augmented \(A\) (RAFALE) & \(6513\pm27\) & \(3.2\pm3.1\) & \(\bm{81.8}\) & \(32.3\pm7.7\) & \textbf{5/5} \\
& Hinge \(C\) & \(6540\pm28\) & \(9.1\pm4.3\) & 128.9 & \(32.3\pm12.1\) & 3/5 \\
& Standard \(U\) & \textcolor{bestfont}{\(\bm{6543\pm11}\)} & \(8.0\pm13.9\) & 147.3 & \(33.7\pm12.7\) & 4/5 \\
\bottomrule
\end{tabular}
\end{table}

\begin{table}[!htbp]
\caption{Actor families with their native regularizers: final reward / cost
(mean \(\pm\) SD, five seeds; feasible seeds in parentheses).  Bold blue:
highest reward; bold black: lowest cost.}
\label{tab:actor-family}
\centering
\scriptsize
\renewcommand{\arraystretch}{0.9}
\setlength{\tabcolsep}{4pt}
\begin{tabular}{lccc}
\toprule
Task & Flow (RAFALE) & Gaussian & GMM \\
\midrule
Hopper   & \({\color{bestfont}\bm{1655\!\pm\!46}}\;/\;\bm{0.3\!\pm\!0.4}\) (5) & \(1408\!\pm\!506\;/\;7.6\!\pm\!17.0\) (4) & \(1645\!\pm\!51\;/\;1.5\!\pm\!2.4\) (5) \\
Walker2d & \(3177\!\pm\!26\;/\;\bm{1.3\!\pm\!1.3}\) (5) & \(3151\!\pm\!15\;/\;3.5\!\pm\!4.0\) (4)  & \({\color{bestfont}\bm{3190\!\pm\!17}}\;/\;1.4\!\pm\!1.1\) (5) \\
Humanoid & \({\color{bestfont}\bm{6513\!\pm\!27}}\;/\;\bm{3.2\!\pm\!3.1}\) (5) & \(6510\!\pm\!51\;/\;8.1\!\pm\!2.6\) (4)  & \(6503\!\pm\!116\;/\;8.5\!\pm\!5.1\) (3) \\
\bottomrule
\end{tabular}
\end{table}

\FloatBarrier

\section{Conclusion and Limitations}
\label{sec:conclusion}

We introduced RAFALE by formulating a constrained flow update as a one-ended
Schr\"odinger bridge, whose reference-process regularizer becomes a computable
kinetic-energy term.  Our analysis shows that, for each
source draw, this path-space problem is exactly an entropy-regularized problem
in action space.  Its Gibbs solution
lowers the probability of actions whose estimated cost exceeds a threshold
set by the multiplier, and its vanishing-noise limit gives the objective of our midpoint
actor.  Across the tested Safety-Gymnasium tasks,
RAFALE achieves competitive reward with mean final cost within budget, and the
ablations support both the augmented objective and the flow actor.

These final-policy results do not ensure safe exploration, since intermediate
policies exceed the budget, and our analysis concerns a single actor update
rather than the coupled actor--critic--dual learner.  Controlling transient
violations, analyzing this coupled learning process, and moving beyond
simulation to real systems are important next steps.  Finally, a Gaussian actor remains
the simpler choice when it already balances reward and safety; the flow
is expected to matter most when the action distribution is multimodal.

\section*{Reproducibility Statement}

\Cref{sec:background,sec:bridge} define the update, and \cref{app:proofs,app:claims}
give the proofs and their assumptions.  \Cref{app:algorithm,app:experiments}
specify the learner, task configurations, baseline implementations, and
evaluation protocols.  We will release the training and evaluation code,
including the baseline re-implementations.

\section*{AI Use Statement}
In this work, we used generative AI tools to help develop the theoretical
framework, to formulate mathematical claims and assist in proving them, to
propose and refine hypotheses, to design and critique the experimental
methodology, to implement the method, the baselines, and the evaluation code,
and to interpret results.  We have not used generative AI tools to generate
synthetic data; translation, dataset cleaning, and qualitative data analysis are
not applicable to this work.  Additionally, we used generative AI tools to search
for related work and to draft and edit sections of the paper.  We have reviewed
all AI-assisted work, and we take
responsibility for the final content of this work, including text, claims, and
artifacts produced with the aid of generative AI.

\bibliographystyle{iclr2027_conference}
\bibliography{references}

\raggedbottom
\makeatletter\setlength{\@fptop}{0pt}\setlength{\@fpsep}{10pt plus 2pt}\setlength{\@fpbot}{0pt plus 1fil}\makeatother
\renewcommand{\topfraction}{0.9}\renewcommand{\bottomfraction}{0.9}\renewcommand{\textfraction}{0.05}\renewcommand{\floatpagefraction}{0.85}
\clearpage
\appendix
\crefalias{section}{appendix}
\crefalias{subsection}{appendix}

\FloatBarrier
\section{Extended Related Work}
\label{app:related}

\paragraph{A common frame.}
A diffusion or flow policy generates an action by transforming a source draw
\(x_0\sim\mu_0\) through a sequence of learned steps; write the result as
\(a=g_\theta(s,x_0)\).  A critic \(Q(s,a)\) scores actions.  In reward-only
RL, most methods aim at the regularized target policy
\begin{equation}
\pi^*(a\mid s)\propto\nu(a\mid s)\exp\big(Q(s,a)/\beta\big),
\label{eq:target-policy}
\end{equation}
which favors high-value actions while staying close to a reference policy
\(\nu\), such as a uniform distribution or the previous policy; \(\beta>0\) is
a temperature \citep{gao2026flowrltaxonomy}.  The methods differ in which
signal of the critic reaches \(\theta\), and four routes are common.  In safe
RL, the negative augmented Lagrangian \(-L_A\) of \cref{eq:aug} plays the role
of \(Q\), so each route extends to the constrained setting.

\paragraph{1) Q-value guidance.}
The target of \cref{eq:target-policy} has a simple score, the gradient of its
log-density:
\[
\nabla_a\log\pi^*(a\mid s)=\nabla_a\log\nu(a\mid s)+\nabla_aQ(s,a)/\beta .
\]
A diffusion policy is trained through its score network, so Q-score matching
trains this network toward the critic's action gradient
\citep{psenka2024qsm}.  The critic thus tells each action in which
direction its value increases.  A diffusion model, however, needs the score of
the target after noise has been added at each step of its chain, and this
noisy score is not \(\nabla_aQ/\beta\).  It must be approximated, learned with
auxiliary networks, or estimated from candidate actions, for example by
importance weighting \citep{gao2026flowrltaxonomy}.

\paragraph{2) Weighted matching.}
These methods turn policy improvement into weighted imitation.  They draw
candidate actions \(a_1,\dots,a_m\), from the current policy or the replay
buffer, weight each by \(w_i\propto\exp(Q(s,a_i)/\beta)\), and train the model
with its usual loss under these weights,
\[
\min_\theta\;\E_s\Big[\textstyle\sum_{i=1}^m w_i\,\ell_\theta(s,a_i)\Big],
\]
where \(\ell_\theta(s,a)\) is the denoising or flow-matching loss of the model
for the action \(a\) \citep{ma2025rsm,li2026rfm,kim2026flag}.  High-value
actions are thus imitated more.  The route fits an intermediate regression target built from
candidate actions and their weights; some estimators, such as those of RFM,
also use the critic's action gradients through Stein control variates
\citep{li2026rfm}.  With
exponential weights, the minimizer of the weighted loss is the target of
\cref{eq:target-policy} \citep{ma2025rsm,gao2026flowrltaxonomy}.  Flow-Based Policy combines
replay-weighted flow matching with a pathwise max-\(Q\) term
\citep{lv2025flowrl}.

\paragraph{3) Policy gradient.}
These methods update the generative policy like any stochastic policy, with a
PPO-style clipped objective
\[
\E\Big[\min\big(r_\theta\hat A,\;\mathrm{clip}(r_\theta,1-\kappa,1+\kappa)\,\hat A\big)\Big],
\qquad r_\theta=\frac{\pi_\theta(a\mid s)}{\pi_{\mathrm{old}}(a\mid s)},
\]
which raises the probability of actions with positive advantage \(\hat A\); the
clipping range \(\kappa\) keeps each update small.  A flow or diffusion policy
has no tractable density, so \(r_\theta\) is replaced by a surrogate: FPO uses
\(\exp(\ell_{\mathrm{old}}(s,a)-\ell_\theta(s,a))\), where \(\ell\) is a Monte
Carlo estimate of the conditional flow-matching loss of the action, and the
update is on-policy \citep{mcallister2025fpo,yi2026fpocontrol}.

\paragraph{4) Reparameterization.}
These methods treat the sampler as a differentiable function
\(a=g_\theta(s,x_0)\) and raise the critic value of the generated action by the
chain rule,
\[
\nabla_\theta Q\big(s,g_\theta(s,x_0)\big)
=\Big(\frac{\partial g_\theta(s,x_0)}{\partial\theta}\Big)^{\!\top}
\nabla_aQ(s,a)\Big|_{a=g_\theta(s,x_0)},
\]
as SAC does for a Gaussian actor \citep{haarnoja2018sac}.  The critic's action
gradient reaches \(\theta\) directly, without an estimated score or weights
\citep{wang2024dacer,celik2025dime,lv2026flac}.  The difficulty is
regularization: the entropy regularization of SAC needs the action log-density, which a
multi-step sampler does not readily provide.  DIME estimates the entropy along
the sampling chain \citep{celik2025dime}, DACER regularizes a tractable proxy
\citep{wang2024dacer}, and FLAC replaces the entropy by the kinetic energy of
the generation path, which needs no density \citep{lv2026flac}.  One-step flow
policies shorten the chain to be differentiated
\citep{li2026flame,zhou2026trfp}, and adjoint matching replaces
backpropagation through it \citep{thilges2026adjoint}.  RAFALE belongs to this
family.  FLAC is the closest prior work: both differentiate a
kinetic-energy-regularized flow path directly, with the reward critic as
the terminal potential there and the augmented Lagrangian \(L_A\) here.

\begin{table}[H]
\centering
\small
\begin{tabular}{@{}>{\raggedright\arraybackslash}p{2.45cm}>{\raggedright\arraybackslash}p{2.95cm}>{\raggedright\arraybackslash}p{3.7cm}>{\raggedright\arraybackslash}p{3.25cm}@{}}
\toprule
Route & Critic signal & What is trained & Extra requirement\\
\midrule
Q-value guidance & \(\nabla_aQ\) & score network, toward \(\nabla_aQ/\beta\) & scores at every noise level\\
Weighted matching & \(Q\) of each candidate (also \(\nabla_aQ\) in RFM) & weighted denoising or flow-matching loss & candidate actions and weights\\
Policy gradient & advantage & clipped ratio objective, on-policy & surrogate for the density ratio\\
Reparameterization & \(\nabla_aQ\) at the generated action & sampler, by backpropagation & density-free regularizer\\
\bottomrule
\end{tabular}
\end{table}

\paragraph{Safe RL with Gaussian actors.}
Primal--dual safe RL enforces an expected cumulative-cost budget by alternating
policy and multiplier updates \citep{altman1999cmdp,achiam2017cpo,stooke2020pid};
alternatives replace the multiplier by state augmentation with the remaining
budget \citep{sootla2022saute} or extract a constrained policy by
KL-regularized variational inference \citep{liu2022cvpo}.  The projected form
of the augmented Lagrangian, its active-set threshold, and the pointwise
multiplier estimate it induces are classical
\citep{rockafellar1973multiplier,bertsekas1982constrained,hintermuller2006pathfollowing},
and CAL applies the projected augmentation to a policy-level expected residual
\citep{wu2024cal}.  State-wise methods such as RCRL and RESPO instead
constrain safety through a learned reachability constraint
\citep{yu2022rcrl,ganai2023respo}.  In the frame above, PID-Lag and RESPO
update Gaussian actors by policy gradient, and CAL and RCRL by
reparameterization (\cref{tab:positioning}).

\paragraph{Safe RL with generative actors.}
Offline, generative policies have been combined with hard state-wise safety:
FISOR extracts a diffusion policy by feasibility-guided weighted regression, a
form of weighted matching \citep{zheng2024fisor}, and EpiFlow and SafeFQL train
flow policies against reachability-derived value functions, by
epigraph-reweighted flow matching and by behavior cloning with a distilled
one-step safe actor, respectively \citep{tayal2026epiflow,tayal2026safefql};
none of these trains online.  Online, ALGD follows the Q-value-guidance route
with the augmented Lagrangian: it takes \(L_A\) of \cref{eq:aug} as the energy
of a Boltzmann target policy, estimates the score of that target at each noise
level by a normalized, energy-weighted Monte Carlo average over candidate
actions, and trains the diffusion score network to match this estimate
\citep{cheng2026algd}.  RAFALE keeps the expected-cost setting and the same
\(L_A\) but follows the reparameterization route: it evaluates \(L_A\) at each
action generated by the flow and passes its gradient to the flow parameters
through the sampler.  We
adopt the augmented Lagrangian for the stability it brings.

\paragraph{Schr\"odinger-bridge view of policy updates.}
A one-ended generalized Schr\"odinger bridge starts a noisy process from a
fixed source distribution, lets its endpoint adapt to an objective evaluated there, and
keeps the whole path close to a reference process, here Brownian motion.
Closeness is measured by a KL divergence between path distributions, whose
value, scaled by the noise level, equals the expected kinetic energy of the
drift by Girsanov's theorem.  The optimal path
distribution reweights the reference by a Gibbs factor of that objective,
and as the noise vanishes the problem becomes deterministic least-energy
transport
\citep{todorov2006lsmdp,leonard2014survey,leonard2012limit,chen2016sb,zhang2022pis};
generalized Schr\"odinger bridge matching adds task-specific state costs while
retaining prescribed endpoint marginals \citep{liu2024gsbm}, whereas the
one-ended formulation used here fixes only the source and shapes the terminal
distribution through a terminal potential.  In reinforcement learning,
FLAC derives kinetic-energy regularization for a reward-only flow actor from a
one-ended formulation, and SoftGAC lifts the endpoint maximum-entropy objective to a path
KL for a reward-only bridge actor, decomposing that KL at the endpoint and
deriving the conditional Gibbs optimizer for a fixed base
\citep{lv2026flac,he2026softgac}.  The static form of the
Schr\"odinger problem, Gibbs tilting, its small-noise limit, the augmented
Lagrangian, and kinetic-energy regularization are thus all prior art.  What is
new here is their combination for the constrained update.  With the Brownian
reference, the endpoint term of this decomposition becomes an explicit
displacement plus a conditional entropy (\cref{thm:action-form}), so the
path-space problem is, for each source draw, an entropy-regularized problem in
action space.  The safety tilt holds
conditionally on each source draw because the source distribution is fixed, and
the resulting map objective is evaluated without approximation by the midpoint
loss of RAFALE.

\FloatBarrier
\section{Proofs and Additional Frozen-Update Results}
\label{app:proofs}

\FloatBarrier
\subsection{Proofs of \texorpdfstring{\cref{thm:action-form,thm:bridge-laws,thm:vanishing}}{Theorems 1 to 3}}
\label{app:proof-bridge}

The proofs use the explicit form of the augmented term.  Expanding the
projected square in \cref{eq:aug} gives
\begin{equation}
\Phi_{\lambda,\rho}(Q_c-h)=\frac{\rho}{2}\,[Q_c-\tau]_+^2-\frac{\lambda^2}{2\rho},
\qquad
\tau=h-\frac{\lambda}{\rho},
\label{eq:penalty-hinge}
\end{equation}
so \(L_A=-Q_r+\frac{\rho}{2}[Q_c-\tau]_+^2\) up to the constant
\(-\lambda^2/(2\rho)\), which is \cref{eq:threshold}.

\begin{theorem}[Formal version of \cref{thm:action-form,thm:bridge-laws,thm:vanishing}]
\label{thm:bridge-laws-formal}
Assume that \(\mathcal S\) is Polish (a complete separable metric space),
that \(Q_r\) and \(Q_c\) are bounded and continuous, and that \(\mu_0\)
has a finite second moment.  Let \(q_A^\varepsilon(\cdot\mid s,x_0)\) and
\(q_R^\varepsilon(\cdot\mid s,x_0)\) denote the terminal distributions of
\cref{eq:bridge-objective} for \(L_A\) and for \(-Q_r\) in place of \(L_A\), conditional
on the source draw \(X_0=x_0\).  Then, for every state \(s\) and every
\(\varepsilon>0\), the optimal value of \cref{eq:bridge-objective} equals the
right-hand side of \cref{eq:free-energy}, where \(q\) ranges over densities with finite
KL divergence from \(\N(x_0,\varepsilon I)\), so that the expected total energy and
the entropy are finite.  For every source draw \(x_0\), the
optimizer of \cref{eq:bridge-objective} is a Gibbs path measure whose
conditional terminal distribution is \cref{eq:gibbs} and which, given both
endpoints, follows the Brownian bridge of the reference; moreover,
\begin{equation}
\frac{dq_A^\varepsilon}{dq_R^\varepsilon}(a\mid s,x_0)
=\frac{w^\varepsilon(s,a)}
      {\E_{a'\sim q_R^\varepsilon(\cdot\mid s,x_0)}\,w^\varepsilon(s,a')},
\label{eq:tilt-formal}
\end{equation}
with \(w^\varepsilon\) as in \cref{eq:safety-tilt}; and the optimal value
converges as \(\varepsilon\downarrow0\) to the value of the least-energy map
problem in \cref{eq:vanishing-noise-map}, the infimum ranging over measurable maps with
finite expected squared displacement, equivalently
\(T(s,\cdot)\in L^2(\mu_0;\mathbb R^{d_a})\) because \(\mu_0\) has a finite second
moment; the infimum is attained, exactly by the
maps satisfying \cref{eq:prox-map} for \(\mu_0\)-almost every \(x_0\).
\end{theorem}

\Cref{eq:tilt-formal} restates \cref{eq:safety-tilt}: the normalizer
\(Z^\varepsilon(s,x_0)\) is the reward-only conditional expectation of
\(w^\varepsilon\), so it depends on \((s,x_0)\) but not on \(a\), and
\(w^\varepsilon\le1\) gives \(Z^\varepsilon\le1\).

\paragraph{Path and endpoint divergence.}
Disintegrating a path distribution at its endpoint gives, whenever
\(\KL(P\Vert R^\varepsilon)<\infty\),
\begin{equation}
\begin{split}
\KL(P\Vert R^\varepsilon)
&=\KL(P_1\Vert R_1^\varepsilon)
+\E_{a\sim P_1}\Big[\KL\big(P(\cdot\mid X_1=a)\,\Vert\,R^\varepsilon(\cdot\mid X_1=a)\big)\Big]\\
&\ge\KL(P_1\Vert R_1^\varepsilon),
\end{split}
\label{eq:endpoint-bound}
\end{equation}
where \(P_1\) and \(R_1^\varepsilon\) are the endpoint distributions.  Bounding
the path divergence therefore bounds the divergence of the action distribution
from the reference endpoint, although the two quantities need not be equal.

\paragraph{Endpoint form.}
Fix \(s\) and \(x_0\), and let \(q\) be the endpoint density of a conditional
path distribution \(P(\cdot\mid s,x_0)\).  Applying \cref{eq:endpoint-bound} to the
conditional distributions, whose reference endpoint is \(\N(x_0,\varepsilon I)\),
gives
\begin{equation}
\E_{P(\cdot\mid s,x_0)}\big[L_A(s,X_1)\big]
+\alpha\varepsilon\,\KL\big(P(\cdot\mid s,x_0)\Vert R_{x_0}^\varepsilon\big)
\;\ge\;\E_{a\sim q}\big[L_A(s,a)\big]
+\alpha\varepsilon\,\KL\big(q\Vert\N(x_0,\varepsilon I)\big),
\label{eq:endpoint-reduction}
\end{equation}
with equality when the paths, given their endpoints, follow the reference
bridges.  Since
\(\alpha\varepsilon\,\KL(q\Vert\N(x_0,\varepsilon I))
=\alpha\,\E_{a\sim q}[\mathcal E(x_0,a)]-\alpha\varepsilon H(q)
+\tfrac{\alpha\varepsilon d_a}{2}\log(2\pi\varepsilon)\), minimizing over \(q\) for each source draw gives \cref{eq:free-energy}, which
proves \cref{thm:action-form}; by the Gibbs variational principle, the minimizer
is the density of \cref{eq:gibbs}.

\paragraph{Problems 2 and 3.}
For a velocity field \(v\), let \(P\) be the path distribution of
\cref{eq:controlled}, with \(P\ll R^\varepsilon\) and finite relative entropy.
Under \(P\),
\begin{equation}
\log\frac{dP}{dR^\varepsilon}
=\frac{1}{\sqrt\varepsilon}\int_0^1 v(s,X_t,t)^\top dW_t^P
+\frac{1}{2\varepsilon}\int_0^1\|v(s,X_t,t)\|^2\,dt,
\label{eq:girsanov-density}
\end{equation}
where \(W^P\) is a Brownian motion under \(P\); both processes start from
\(\mu_0\), so no initial term appears.  Finite expected kinetic energy makes the
stochastic integral a mean-zero martingale, and taking expectations gives
\cref{eq:girsanov} \citep{leonard2014survey}.  The objectives of
\cref{prob:bridge,prob:control} therefore agree for every admissible \(v\).
Their infima agree as well, because the Gibbs optimizer constructed below is
induced by the Markov drift
\(v^*(s,x,t)=\varepsilon\nabla_x\log\psi^\varepsilon(s,x,t)\) with
\(\psi^\varepsilon(s,x,t):=\E\exp[-L_A(s,x+\sqrt\varepsilon\,W_{1-t})/(\alpha\varepsilon)]\),
a Doob transform of the reference.  Its terminal density ratio is bounded above
and away from zero, so it has finite relative entropy and hence finite
energy.

\begin{proof}
Fix the state \(s\); it enters only through \(L_A\).  All feasible
path measures share the source marginal \(\mu_0\), and disintegration with
respect to \(X_0=x_0\) gives
\begin{equation}
\KL(P\Vert R^\varepsilon)
=
\E_{x_0\sim\mu_0}
\KL\!\left(P(\cdot\mid s,x_0)\Vert R_{x_0}^\varepsilon\right).
\label{eq:entropy-chain}
\end{equation}
Here \(R_{x_0}^\varepsilon\) is the conditional Brownian path measure given
\(X_0=x_0\); the fixed state \(s\) is carried as a parameter.

For each source draw, the conditional Gibbs variational identity gives
\begin{align}
&\E_{P(\cdot\mid s,x_0)}L_A(s,X_1)
+\alpha\varepsilon
\KL\!\left(P(\cdot\mid s,x_0)\Vert R_{x_0}^\varepsilon\right)
\nonumber\\
&\quad=
-\alpha\varepsilon\log Z_A(s,x_0)
+\alpha\varepsilon
\KL\!\left(P(\cdot\mid s,x_0)\Vert P_A^*(\cdot\mid s,x_0)\right),
\label{eq:gibbs-identity}
\end{align}
where
\(Z_A(s,x_0)=\E_{R_{x_0}^\varepsilon}
\exp[-L_A(s,X_1)/(\alpha\varepsilon)]\) and
\begin{equation}
\frac{dP_A^*(\cdot\mid s,x_0)}{dR_{x_0}^\varepsilon}
=
\frac{\exp[-L_A(s,X_1)/(\alpha\varepsilon)]}{Z_A(s,x_0)}.
\label{eq:conditional-change}
\end{equation}
The last term is nonnegative and vanishes only at \(P_A^*\), so \(P_A^*\) is
the optimizer.  Since the endpoint of \(R_{x_0}^\varepsilon\) is distributed as
\(\N(x_0,\varepsilon I)\), the endpoint of \(P_A^*\) has the density of
\cref{eq:gibbs}, which is part (a) of \cref{thm:bridge-laws}.

To obtain the reward-relative tilt, divide the conditional terminal density for
\(L_A=-Q_r+\Phi_{\lambda,\rho}(Q_c-h)\) by the corresponding reward-only
density.  The reference and reward terms cancel.  Using
\cref{eq:penalty-hinge}, the constant independent of the terminal state
\(-\lambda^2/(2\rho)\) also cancels against the conditional normalizer,
proving \cref{eq:tilt-formal} and hence \cref{eq:safety-tilt}.

For the small-noise limit, conditional on \((s,x_0)\), the Gaussian Laplace principle
gives
\begin{equation}
-\alpha\varepsilon\log
\E_{X_1\sim\N(x_0,\varepsilon I)}
\exp\!\left[-\frac{L_A(s,X_1)}{\alpha\varepsilon}\right]
\longrightarrow
\inf_a\left\{
L_A(s,a)+\frac{\alpha}{2}\|a-x_0\|^2
\right\}.
\label{eq:laplace-limit}
\end{equation}
Boundedness permits integration over \(\mu_0\), and measurable selection or
minimizing sequences identify the integrated pointwise infimum with the map
infimum of \cref{eq:vanishing-noise-map}.  The pointwise infimum is attained, because \(L_A(s,\cdot)\) is continuous
and the displacement term grows quadratically, and a measurable selection of
minimizers exists by the measurable maximum theorem.  Conversely, a map attains the infimum of
\cref{eq:vanishing-noise-map} only if it satisfies \cref{eq:prox-map} for
\(\mu_0\)-almost every \(x_0\), since the integrand is never smaller than its
pointwise infimum.  When the minimizers are not unique, this value limit does
not require the positive-noise conditional distributions to converge to point
masses; they can converge to mixtures over the minimizers.  A minimizing selection \(T\) has finite expected squared
displacement: comparing with \(a=x_0\) gives
\(\frac{\alpha}{2}\|T(s,x_0)-x_0\|^2\le L_A(s,x_0)-\inf_aL_A(s,a)\), which is
bounded because the critics are.  Conversely, \cref{app:map-lift}
constructs, for any map, positive-noise path measures whose cost converges to
its map objective.  The limit in \cref{eq:laplace-limit} is the standard terminal-cost
Schr\"odinger-to-proximal limit
\citep{leonard2012limit,chen2016sb,li2023kernel}.
\end{proof}

\Cref{app:solver-identity} restricts this limit to the implemented solver map,
and \cref{app:tilt-consequences} derives a consequence of the tilt for the
estimated cost.

\FloatBarrier
\subsection{The Finite-Step Objective Identity}
\label{app:solver-identity}

\begin{corollary}[Objective for the midpoint map]
\label{cor:solver-objective}
Restricted to the midpoint maps of \cref{eq:midpoint-main}, the map objective
in \cref{eq:vanishing-noise-map}, averaged over replay states, equals the
actor loss \(\mathcal L(\theta)\) in \cref{eq:actor-loss}.
\end{corollary}

\begin{proof}

The midpoint map of \cref{eq:midpoint-main} gives
\(X_1-x_0=\delta_\theta(s,x_0)\), so its kinetic energy is \(K_\theta\) in
\cref{eq:exact-kinetic}; this is the one-interval case of the discretized
kinetic energy in \cref{eq:kinetic-statistic}.  Evaluating the deterministic objective
in \cref{eq:vanishing-noise-map}, averaged over \(s\sim\mathcal D\), at the midpoint map \(T(s,x_0)=X_1\), and
substituting the definition of
\(L_A\), gives exactly \cref{eq:actor-loss}.  Because the critics and coefficients
are frozen during the actor step, automatic differentiation through both velocity
evaluations computes this finite graph's derivative at differentiability points;
at a nonsmooth tie, it follows the subgradient selected by the implementation.
\end{proof}

The identity holds for the implemented discrete map whether or not the solver
tracks the continuous ODE trajectory.

\FloatBarrier
\subsection{Additional Consequence of the Safety Tilt}
\label{app:tilt-consequences}

Fix \((s,x_0)\), let \(a\sim q_R^\varepsilon(\cdot\mid s,x_0)\), and set
\(C=Q_c(s,a)\) and
\[
w(C)=\exp[-\rho(C-\tau)_+^2/(2\alpha\varepsilon)].
\]
The weight is nonincreasing.  We write \(\E_R\) and \(\E_A\) for
expectations under \(q_R^\varepsilon(\cdot\mid s,x_0)\) and
\(q_A^\varepsilon(\cdot\mid s,x_0)\), respectively.  For every
integrable nondecreasing \(\varphi\),
\begin{equation}
\E_A\varphi(C)-\E_R\varphi(C)
=
\frac{\operatorname{Cov}_R(\varphi(C),w(C))}{\E_Rw(C)}
\le0.
\label{eq:frozen-fosd}
\end{equation}
Thus, for each source draw, the tilt shifts probability toward actions with
lower estimated cost: the estimated cost under the augmented conditional
terminal distribution is first-order stochastically no larger than under the
reward-only one, so its mean and the probability of exceeding any cost level
are no larger under the tilt.  The tilt also matches the pathwise update.  The action gradient of \(L_A\) is
\begin{equation}
\nabla_aL_A(s,a)=-\nabla_aQ_r(s,a)+m_A(s,a)\,\nabla_aQ_c(s,a),
\qquad m_A(s,a):=\rho\,[Q_c(s,a)-\tau]_+ ,
\label{eq:weight-A}
\end{equation}
whereas the standard Lagrangian weights \(\nabla_aQ_c\) by \(\lambda\) for every
action.  The weight \(m_A\) is zero below \(\tau\), grows linearly above it, and
equals \(\lambda\) at \(Q_c=h\); it is also the weight in the tilt, since
\(-\alpha\varepsilon\,\nabla_a\log w^\varepsilon(s,a)=m_A(s,a)\,\nabla_aQ_c(s,a)\).

\FloatBarrier
\subsection{Constructive Positive-Noise Lift of a Terminal Map}
\label{app:map-lift}

For a finite-energy measurable map \(T\), let \(P_T^\varepsilon\) be the path measure
of
\begin{equation}
dX_t=(T(s,X_0)-X_0)\,dt+\sqrt\varepsilon\,dW_t,
\qquad t\in[0,1],
\qquad X_0\sim\mu_0.
\end{equation}
Then \(X_1=T(s,X_0)+\sqrt\varepsilon W_1\), so the joint terminal distribution converges
weakly to the pushforward under \((s,x_0)\mapsto(s,T(s,x_0))\), while Girsanov's identity
gives
\begin{equation}
\alpha\varepsilon\KL(P_T^\varepsilon\Vert R^\varepsilon)
=
\frac{\alpha}{2}\E\|T(s,X_0)-X_0\|^2.
\end{equation}
Because \(L_A\) is bounded and continuous, weak convergence also gives
convergence of the expected terminal potential, so the cost of
\(P_T^\varepsilon\) converges to the map objective of \(T\).  The drift depends on the source draw, so
\(P_T^\varepsilon\) is a path distribution admissible in \cref{prob:bridge}, not the
path distribution of a Markov velocity field.

\FloatBarrier
\subsection{Injectivity of the Vanishing-Noise Map}
\label{app:injective}

Let \(T\) attain the infimum of \cref{eq:vanishing-noise-map}, and suppose that
\(L_A(s,\cdot)\) is differentiable.  For \(\mu_0\)-almost every \(x_0\), the
action \(a=T(x_0)\) minimizes \(L_A(s,a)+\alpha\,\mathcal E(x_0,a)\), so its
first-order condition \(\alpha(a-x_0)=-\nabla_aL_A(s,a)\) gives
\begin{equation}
x_0=G(a),\qquad G(a):=a+\nabla_aL_A(s,a)/\alpha .
\label{eq:prox-stationarity}
\end{equation}
The map \(G\) does not depend on \(x_0\), and \(G(T(x_0))=x_0\), so \(T\) is
injective: distinct source draws are sent to distinct actions.  If
\(\nabla_aL_A(s,\cdot)\) is \(L\)-Lipschitz, then \(G\) is
\((1+L/\alpha)\)-Lipschitz, and hence
\begin{equation}
\big\|T(x_0)-T(x_0')\big\|\;\ge\;\frac{\|x_0-x_0'\|}{1+L/\alpha}.
\label{eq:no-collapse}
\end{equation}
The coefficient \(\alpha\) thus bounds how much \(T\) can contract the source
distribution.  The bound vanishes as \(\alpha\to0\), where the map objective
reduces to \cref{prob:direct}, whose minimizers do not depend on \(x_0\).  The
argument concerns an exact minimizer of \cref{eq:vanishing-noise-map}; the
midpoint actor approximates such a map within its parameterized class, and the
minimum of two reward critics can make \(L_A\) nondifferentiable where they tie.

\FloatBarrier
\section{Assumptions and Scope}
\label{app:claims}

The bridge analysis concerns one actor update, with the critics, the replay
distribution, and the coefficients \((h,\lambda,\rho,\alpha)\) held fixed.  The
augmented Lagrangian is evaluated with learned critics, so its activation
threshold is not a certified state-wise safety boundary, and ordering
discounted cost estimates does not by itself establish feasibility for the
undiscounted episode budget.  The full actor--critic--dual learning process is
not a single bridge problem, and the results do not imply coupled convergence
or safe exploration.

Girsanov's identity applies at positive noise, when the controlled and
reference processes share the diffusion coefficient.  A deterministic flow has
zero quadratic variation, whereas a Brownian path does not, so their path
distributions are mutually singular and the path divergence is infinite.  The
deterministic objective is therefore obtained as the small-noise limit of the
scaled-divergence problems, not by assigning a finite Brownian path divergence
to the noise-free flow.  Source sampling still induces an action distribution,
but the displacement term supplies geometric regularization rather than a
lower bound on its entropy.

The finite-step identity of \cref{app:solver-identity} concerns the midpoint map
and its discrete energy.  It does not require the numerical step to reproduce a
continuous optimal trajectory, nor does it imply that the parameterized actor
attains the infimum over all measurable maps.  Backpropagation differentiates
the implemented graph at points of differentiability and uses the subgradient
selected by the implementation at nonsmooth ties.

\section{Practical Actor--Critic Implementation}
\label{app:algorithm}

The learner of \cref{alg:method} combines the flow update with off-policy
critics, an adaptive kinetic coefficient, and a multiplier driven by completed
episodes.

\begin{algorithm}[!htbp]
\caption{RAFALE: actor--critic learning with a pathwise flow-policy update}
\label{alg:method}
\begin{algorithmic}[1]
\small
\Require source distribution \(\mu_0\), budget \(b\), horizon \(H\), coefficient \(\rho\), kinetic target \(K_{\mathrm{tgt}}\), replay buffer \(\mathcal D\)
\State \(h\gets\kappa_H b\); initialize \(v_\theta\), critics \(Q_{r,1},Q_{r,2},Q_c\) and target critics \(Q^-_{r,1},Q^-_{r,2},Q^-_c\), \(z=\lambda=0\), \(\log\alpha\), \(\mathcal W\gets\varnothing\)
\For{each environment step}
    \State draw \(x_0\sim\mu_0\), compute \(X_1=x_0+\delta_\theta(s,x_0)\) by \cref{eq:midpoint-main}, and execute \(a=\tanh(X_1)\)
    \State store \((s,a,r,c,s',d)\) in \(\mathcal D\); when an episode completes, append its cost to \(\mathcal W\)
    \If{an update cycle is due}
        \For{each gradient update}
            \State sample \((s,a,r,c,s',d)\sim\mathcal D\) and \(x_0,x_0'\sim\mu_0\); set \(X_1'=x_0'+\delta_\theta(s',x_0')\), \(K'=K_\theta(s',x_0')\)
            \State regress \(Q_{r,1},Q_{r,2}\) and \(Q_c\) onto
            \Statex \hspace{\algorithmicindent}\hspace{\algorithmicindent}\hspace{\algorithmicindent}
            \(y_r=r+\gamma(1-d)\big[\min_i Q^-_{r,i}(s',X_1')-\alpha K'\big]\), \quad
            \(y_c=c+\gamma(1-d)\,Q^-_c(s',X_1')\)
            \If{the update is an actor step (every second update)}
                \State compute \(X_1=x_0+\delta_\theta(s,x_0)\) by \cref{eq:midpoint-main} and \(K_\theta(s,x_0)\)
                \State update \(\theta\) by backpropagation through both velocity evaluations on
                \Statex \hspace{\algorithmicindent}\hspace{\algorithmicindent}\hspace{\algorithmicindent}\hspace{\algorithmicindent}
                \(\ell(\theta)=-\min_i Q_{r,i}(s,X_1)+\Phi_{\lambda,\rho}\big(Q_c(s,X_1)-h\big)+\alpha K_\theta(s,x_0)\)
                \State update \(\log\alpha\) on \(\mathcal L_\alpha\) of \cref{eq:log-alpha-loss}, and update the target critics
            \EndIf
        \EndFor
    \EndIf
    \If{a multiplier update is due}
        \State compute the mean episodic cost \(\bar C_{\mathcal W}\) over \(\mathcal W\) and update \(z\) and \(\lambda\) by the projected PI rule of \cref{eq:dual-update}
    \EndIf
\EndFor
\end{algorithmic}
\end{algorithm}

\paragraph{Action squashing.}
The main text treats the critics as functions of the solver output.  The implementation squashes solver outputs into the action box with
\(\tanh\) before execution and critic evaluation, so \(Q_r\) and \(Q_c\) in the
main text denote the learned critics composed with the squashing map.
The composition preserves boundedness and continuity, so the assumptions of
\cref{thm:bridge-laws,thm:vanishing} still hold.

\paragraph{Solver and kinetic energy.}
The implementation integrates the flow of \cref{eq:flow-policy} with the
single midpoint step of \cref{eq:midpoint-main}, which uses one velocity evaluation to construct
the midpoint and a second to obtain the displacement; its kinetic energy is \(K_\theta\) in
\cref{eq:exact-kinetic}.  For a general solver whose endpoint update has the
form
\begin{equation}
X_1-x_0=\sum_{\ell=1}^{N}\omega_\ell\,v_\theta(s,X_{t_\ell},t_\ell),
\qquad \omega_\ell\ge0,\qquad \sum_{\ell=1}^{N}\omega_\ell=1,
\label{eq:consistent-quadrature}
\end{equation}
the discretized kinetic energy formed with the same nodes and weights, which
generalizes \cref{eq:exact-kinetic},
\begin{equation}
K_\theta(s,x_0)=\frac12\sum_{\ell=1}^{N}\omega_\ell
\|v_\theta(s,X_{t_\ell},t_\ell)\|^2
\label{eq:kinetic-statistic}
\end{equation}
is a quadrature estimate of the continuous path energy
\(\frac12\int_0^1\|v_\theta(s,X_t,t)\|^2dt\), and Jensen's inequality with
\cref{eq:consistent-quadrature} gives
\begin{equation*}
K_\theta(s,x_0)\ge\frac12\Big\|\sum_{\ell=1}^{N}\omega_\ell\,
v_\theta(s,X_{t_\ell},t_\ell)\Big\|^2=\frac12\|X_1-x_0\|^2=\mathcal E(x_0,X_1),
\end{equation*}
the energy of \cref{thm:vanishing}.  With one midpoint interval this
kinetic energy equals half the squared displacement, \cref{eq:exact-kinetic},
which is the identity behind \cref{cor:solver-objective}; for other
consistent quadratures the actor loss upper-bounds the state-averaged map
objective of \cref{eq:vanishing-noise-map}.

\paragraph{Budget conversion and multiplier schedule.}
The environment budget \(b\) is an undiscounted episode total, whereas the
cost critic uses discount \(\gamma\) and bootstraps at time-limit
truncations.  The
nominal conversion used in \cref{sec:background,sec:practical} is
\begin{equation}
h=\kappa_Hb,\qquad
\kappa_H=\frac{1-\gamma^{H}}{(1-\gamma)H},
\label{eq:budget-conversion}
\end{equation}
a scale calibration rather than an unbiased estimate of the remaining
budget.  Observed episodes,
not the critic, drive the multiplier: the actor uses \(h\) in \(L_A\), but critic error does not
enter the dual signal.  With \(\bar C_{\mathcal W}\) the mean raw cost over a
recent window \(\mathcal W\) of completed episodes and
\(e=\kappa_H(\bar C_{\mathcal W}-b)\) the residual in cost-critic units, the
multiplier follows the projected proportional--integral update
\begin{equation}
z^{+}=\Pi_{[0,\lambda_{\max}]}\big(z+\eta_\lambda e\big),
\qquad
\lambda^{+}=\Pi_{[0,\lambda_{\max}]}\big(z^{+}+\eta_p e\big),
\label{eq:dual-update}
\end{equation}
where \(\Pi_{[0,\lambda_{\max}]}\) projects onto \([0,\lambda_{\max}]\) for a
fixed cap \(\lambda_{\max}\), \(z\) is the integral state,
\(\eta_\lambda>0\) its stepsize, and \(\eta_p\ge0\) an optional proportional
gain \citep{stooke2020pid}; \(\eta_p=0\) recovers integral-only feedback and
is used except in the Humanoid protocol, where \(\eta_p=.05\).  Completed
episodes determine the multiplier \(\lambda\), while the cost critic
determines its effect on each generated action through the weight
\(m_A\) of \cref{eq:weight-A}.  The region that receives no cost
gradient is therefore a property of the learned \(Q_c\), \(\lambda\), and
\(\rho\) together, not of the projected-quadratic form itself.  

\paragraph{Critic targets.}
The learner uses two reward critics \(Q_{r,i}\), \(i\in\{1,2\}\), regularized by the kinetic energy,
and one cost critic, together with target networks updated by exponential
moving averages, denoted by a superscript minus.  For a replay transition
\((s,a,r,c,s',d)\), with \(d=1\) only for a true environment termination
(time-limit truncations continue to bootstrap), we draw \(x_0'\sim\mu_0\),
generate \(a'=X_1'\) at \(s'\) by the midpoint step of \cref{eq:midpoint-main},
and write \(K'=K_\theta(s',x_0')\).  With \(\gamma=.99\), the targets are
\begin{align}
y_r
&=
r+\gamma(1-d)\Big[\min_{i\in\{1,2\}}Q_{r,i}^{-}(s',a')-\alpha K'\Big],
\label{eq:reward-target}\\
y_c
&=
c+\gamma(1-d)\,Q_c^{-}(s',a'),
\label{eq:cost-target}
\end{align}
the reward target being that of FLAC \citep{lv2026flac}.  Both critics
minimize the squared error to their targets; target actions come from the
current actor, and no target actor is maintained.

\paragraph{Kinetic-energy coefficient.}
Adam updates \(\log\alpha\) on \cref{eq:log-alpha-loss}, following the
automatic energy tuning of FLAC \citep{lv2026flac}, and a projection keeps
\(\alpha\ge0.003\).  The realized kinetic energy thus adapts \(\alpha\),
replay transitions train the critics, and completed episodes set \(\lambda\).

\paragraph{Completed-episode feedback.}
The episode window contains at most the ten most recent completed
episodes.  On Swimmer and Walker2d the integral state is warm-started: at the
first multiplier update after the dual warm-up, \(z\) is set to
\(z_{\mathrm w}\) (\cref{tab:dual-configs}) before the update is applied.  Only
the multiplier waits for this window; critic and actor updates continue at the
transition level.

\FloatBarrier
\section{Experimental Details and Additional Results}
\label{app:experiments}

\FloatBarrier
\subsection{Tasks and Evaluation Protocol}
\label{app:configs}

We evaluate RAFALE, CAL, RCRL, and ALGD on the same fixed panel of episode
layouts and policy seeds, with one stochastic action sample per decision: 50
episodes per checkpoint on the velocity tasks and 200 on CarGoal.  RESPO uses
this panel on Humanoid and CarGoal and its own evaluator, with 50 stochastic
episodes per checkpoint, on the other tasks.  Tables report each
method's checkpoint at its step budget (\cref{tab:final-configs}
for the off-policy methods, 9M--10M for RESPO); no checkpoint is selected by
performance.  All seeds are reported.  In the learning-curve figures, each
seed's evaluations are smoothed by a centered moving average over 10\% of the
task's step budget (0.2 decades on the logarithmic axis of \cref{fig:curves-log}),
with the window truncated symmetrically at both ends; all numbers in the
tables use the unsmoothed evaluations.  RAFALE uses episode horizon \(H=1000\), \(\rho=0.1\),
one explicit-midpoint step (two velocity-network evaluations) per action, and
action repeat one.  The update-to-data ratio (UTD) is the number of gradient
updates per environment transition, and ``anneal start'' is the environment
step at which cosine decay of the actor learning rate begins, while the critic,
kinetic-coefficient, and multiplier learning rates stay constant.

\FloatBarrier
\subsection{RAFALE Settings}

\Cref{tab:common-configs} lists the settings shared by all tasks,
\cref{tab:final-configs} the task-level settings, and \cref{tab:dual-configs}
the multiplier settings.

\begin{table}[!htbp]
\caption{Common training settings; LR denotes learning rate.}
\label{tab:common-configs}
\centering
\scriptsize
\setlength{\tabcolsep}{3.8pt}
\begin{tabular}{lrlr}
\toprule
Setting & Value & Setting & Value \\
\midrule
Batch size & 256 & Replay capacity & 1M \\
Hidden network & \(2\times256\) & Gradient-norm cap & 10 \\
Actor/critic/\(\alpha\) LR & \(3{\times}10^{-4}\) & Initial \(\log\alpha\) & \(-2\) \\
Environment warmup \(t_0\) & 5k & Dual warmup \(t_\lambda\) & 200k \\
Update cycle & 16 env.\ steps & Updates per cycle \(n_u\) & \(16\times\mathrm{UTD}\) \\
Policy delay & 2 updates & Dual cadence \(d_\lambda\) & 2000 env.\ steps \\
Target smoothing coefficient & 0.1 & Episode window & 10 completed \\
Kinetic target & \(1.125d_a\) & Minimum \(\alpha\) & 0.003 \\
Final actor-LR ratio & 0.05 & Base noise & \(\mathcal N(0,I)\) clipped elementwise to \([-1,1]\) \\
\bottomrule
\end{tabular}
\end{table}

\begin{table}[!htbp]
\caption{Task-level settings for the reported configurations.}
\label{tab:final-configs}
\centering
\scriptsize
\setlength{\tabcolsep}{4.4pt}
\begin{tabular}{lrrrr}
\toprule
Task & \(b\) & Steps & UTD & Anneal start \\
\midrule
Swimmer & 10 & 1M & 1 & 700k \\
HalfCheetah & 10 & 1M & 1 & 700k \\
Hopper & 10 & 2M & 2 & 1.7M \\
Walker2d & 10 & 1M & 1 & 700k \\
Ant & 10 & 3M & 1 & 2.7M \\
Humanoid & 10 & 3M & 2 & 2.4M \\
CarGoal & 25 & 1M & 1 & 400k \\
\bottomrule
\end{tabular}
\end{table}

\begin{table}[!htbp]
\caption{Task-specific multiplier settings used by the reported full method.}
\label{tab:dual-configs}
\centering
\scriptsize
\setlength{\tabcolsep}{6pt}
\begin{tabular}{lrrrr}
\toprule
Task & \(\eta_\lambda\) & \(\eta_p\) & \(\lambda_{\max}\) & Warm-start value \(z_{\mathrm w}\) \\
\midrule
Swimmer & .001 & 0 & 0.5 & 0.3999828 \\
HalfCheetah & .001 & 0 & 4.2 & 0 \\
Hopper & .001 & 0 & 6.7 & 0 \\
Walker2d & .001 & 0 & 16.5 & 1.654 \\
Ant & .001 & 0 & 15.0 & 0 \\
Humanoid & .001 & .05 & 9.8 & 0 \\
CarGoal & .004 & 0 & 0.5 & 0 \\
\bottomrule
\end{tabular}
\end{table}

\subsection{Baselines}
\label{app:baselines}

\paragraph{RESPO.}
We run the official implementation at \(b=10\) for 10M steps (9M on Humanoid
and CarGoal), five seeds, evaluating 50 episodes every 180k steps (every 250k on Humanoid).  The
Humanoid run uses tanh-squashed actions, observation normalization, and a
policy learning rate of \(10^{-4}\); the CarGoal run uses cost limit 25 and
its endpoint is rescored over 200 episodes.

\paragraph{CAL.}
We re-implement the official release in JAX and verify the port against it
numerically.  Settings follow the release defaults: minibatches of 12
transitions, UTD 10, hidden size 256, learning rate \(3\times10^{-4}\),
\(\gamma=\gamma_c=0.99\), and target smoothing 0.005; the augmented
coefficient \(c_{\mathrm{CAL}}\) and the other per-task settings are listed
in \cref{tab:cal-configs}.  Departures from the release: critic layer
normalization (all tasks but Humanoid), bootstrap masks at termination and
truncation, a per-step cost target
\(d_{\mathrm{CAL}}=b(1-\gamma_c^{H})/((1-\gamma_c)H)\) for the cost critic and
the dual, one sampled action per decision at evaluation, and in-loop
evaluation rollouts that are not counted in the step axis.

\begin{table}[!htbp]
\caption{CAL settings per task.  \(N_{Q_c}\) is the cost-critic ensemble size and \(k_{\mathrm{CAL}}\) the
weight of its standard deviation in CAL's cost estimate.  Mask ``t/t''
denotes a cost-critic bootstrap mask of zero on termination or truncation;
``1'' a mask forced to one.}
\label{tab:cal-configs}
\centering
\scriptsize
\setlength{\tabcolsep}{4pt}
\begin{tabular}{lccccccc}
\toprule
Setting & Swimmer & HalfCheetah & Hopper & Walker2d & Ant & Humanoid & CarGoal \\
\midrule
Steps & 1M & 1M & 2M & 1M & 3M & 3M & 1M \\
\(c_{\mathrm{CAL}}\) & 10 & 10 & 100 & 100 & 100 & 1000 & 10 \\
\(k_{\mathrm{CAL}}\) & 0.5 & 0.5 & 0.5 & 0.5 & 0.5 & 0 & 0.5 \\
\(N_{Q_c}\) & 8 & 8 & 8 & 8 & 8 & 4 & 4 \\
Max-over-heads target & no & no & no & no & no & yes & no \\
Critic layer norm & yes & yes & yes & yes & yes & no & yes \\
Cost-critic mask & t/t & t/t & t/t & t/t & 1 & t/t & 1 \\
\(d_{\mathrm{CAL}}\) & 1.000 & 1.000 & 1.000 & 1.000 & 1.000 & 1.000 & 2.500 \\
\bottomrule
\end{tabular}
\end{table}

\paragraph{RCRL.}
We re-implement the official RAC learner in JAX and verify the port against
the official loss definitions, with three departures: the reward critic masks
its bootstrap at true terminations, the safety target bootstraps through
time-limit truncations, and the two targets use separate next-action samples.
RCRL trains on a state-wise zero-violation constraint rather than a budget,
so \(b\) only judges its endpoints.  The reference configuration transfers
the official train-script values (\(\gamma_h=1\), linear learning-rate decay,
multiplier updated every 12 steps and unbounded, target entropy
\(-|\mathcal A|\), one update per environment step); per-task settings are
listed in \cref{tab:rcrl-configs}, and evaluation uses one tanh-Gaussian
sample per decision.

\begin{table}[!htbp]
\caption{RCRL settings per task.}
\label{tab:rcrl-configs}
\centering
\scriptsize
\setlength{\tabcolsep}{3pt}
\begin{tabular}{lccccccc}
\toprule
Setting & Swimmer & HalfCheetah & Hopper & Walker2d & Ant & Humanoid & CarGoal \\
\midrule
Steps & 1M & 1M & 2M & 1M & 3M & 3M & 1M \\
\(\gamma_h\) & 0.99 & 1.0 & 1.0 & 1.0 & 1.0 & 1.0 & 1.0 \\
Learning rates & decay (1M) & decay (3M) & constant & decay (3M) & decay (3M) & decay (3M) & decay (3M) \\
Multiplier update & every 12 & every 12 & every step & every 12 & every 12 & every 12 & every 12 \\
Multiplier cap & none & none & 100 & none & none & none & none \\
Safety-critic layer norm & yes & no & no & no & no & no & no \\
Checkpoint & end of run & 1M of 3M & 2M of 3M & 1M of 3M & end of run & end of run & 1M of 3M \\
\bottomrule
\end{tabular}
\end{table}

\paragraph{ALGD.}
We re-implement the released code in JAX and verify the port against it
numerically, with four proposal candidates, \(\beta=1\), five diffusion
steps, UTD 2, and \(\rho=1\) on the velocity tasks and \(\rho=0.1\) on
CarGoal (\(b=25\), 200 evaluation episodes per checkpoint).

\subsection{Ablation Protocols}
\label{app:ablation-protocol}

\paragraph{Lagrangian ablation.}
The three variants of \cref{sec:ablations} differ only in the Lagrangian
inside the actor loss: Standard \(U\) uses \(L\) of \cref{eq:linear-step},
Augmented \(A\) (RAFALE) uses \(L_A\) of \cref{eq:aug}, and Hinge \(C\) uses
\begin{equation}
L_C(s,a):=-Q_r(s,a)+\lambda\,[Q_c(s,a)-\tau]_+ ,
\label{eq:penalty-variants}
\end{equation}
which shares the threshold \(\tau=h-\lambda/\rho\) with \(L_A\).  The three
constraint terms differ only in their derivatives with respect to \(Q_c\),
\begin{equation}
m_U=\lambda,
\qquad
m_C=\lambda\,\mathbf 1\{Q_c>\tau\},
\qquad
m_A=\rho\,[Q_c-\tau]_+ ,
\label{eq:three-coefficients}
\end{equation}
which multiply \(\nabla_aQ_c\) in the action gradient, as in
\cref{eq:weight-A}; all three equal \(\lambda\) at \(Q_c=h\).  All
variants use the same learner settings and the same completed-episode
multiplier update; each run learns its own critics, replay distribution, and
multiplier trajectory.

\paragraph{Evaluation metrics.}
Let \(C_i(n)\) denote the evaluated mean episodic cost of training seed \(i\)
at environment step \(n\).  Peak is the largest across-seed mean checkpoint
cost, \(\max_n\frac{1}{5}\sum_{i=1}^5 C_i(n)\).
For each seed, Excess is \(10^{-6}\int[C_i(n)-b]_+\,dn\), computed by the
trapezoidal rule over all checkpoints of that run; \cref{tab:gate-ablation-full}
reports its mean and standard deviation across seeds.  Both metrics use
the costs of evaluated policies, not those incurred while collecting training
transitions.  A seed is feasible when its final mean evaluation cost is at
most \(b\).

\paragraph{Policy class.}
The flow, tanh-Gaussian, and eight-component tanh-GMM actors share the critic
architecture, augmented Lagrangian, multiplier update, number of environment steps, and
evaluation protocol.  The Gaussian and GMM use log-density regularization,
while the flow uses kinetic-energy regularization.  This comparison evaluates
the complete actor configurations; it does not isolate the policy family
from its regularizer.

\FloatBarrier
\subsection{Additional Results}
\label{app:more-results}

\paragraph{Seed-level feasibility.}
\Cref{tab:feasible-seeds} gives, for the endpoints of \cref{tab:main-results}, the
number of seeds whose final mean cost is within budget.

\begin{table}[!htbp]
\caption{Seeds, out of five, whose final mean episodic cost is at or below the
budget, for the endpoints of \cref{tab:main-results}.  RAFALE's HalfCheetah
seed 0 ends at 10.9 against a budget of 10.}
\label{tab:feasible-seeds}
\centering
\scriptsize
\setlength{\tabcolsep}{5pt}
\begin{tabular}{lccccc}
\toprule
Task & RESPO & CAL & RCRL & ALGD & RAFALE \\
\midrule
Swimmer     & 5 & 1 & 4 & 5 & 5 \\
HalfCheetah & 2 & 2 & 4 & 4 & 4 \\
Hopper      & 4 & 1 & 5 & 1 & 5 \\
Walker2d    & 2 & 4 & 4 & 2 & 5 \\
Ant         & 5 & 4 & 4 & 1 & 5 \\
Humanoid    & 5 & 5 & 4 & 5 & 5 \\
CarGoal     & 4 & 5 & 3 & 4 & 5 \\
\bottomrule
\end{tabular}
\end{table}

\paragraph{Learning curves against baselines.}
\Cref{fig:curves-offpolicy} in \cref{sec:main-results} compares the four
off-policy methods, which share RAFALE's step budgets; the upper block holds the
1M-step tasks and the lower block the longer ones.  RESPO trains for up to 10M steps; \cref{fig:curves-log} adds
it on a logarithmic step axis, with the dotted vertical line marking the step budget of the other methods, and \cref{tab:respo-matched} tabulates its
values at RAFALE's step budget.

\begin{figure}[!htbp]
\centering
\includegraphics[width=\textwidth]{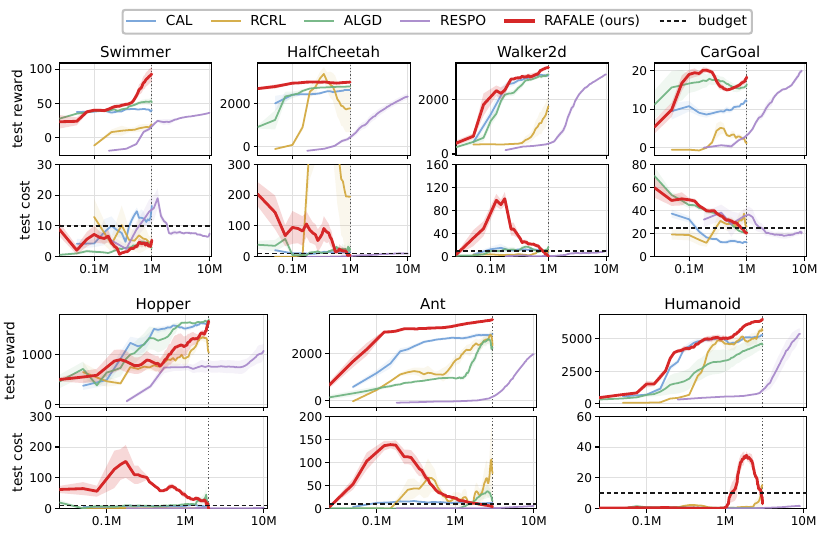}
\caption{All five methods under the convention of \cref{fig:curves-offpolicy},
with environment steps on a logarithmic axis so that RESPO's 10M-step
trajectories fit alongside the others; the dotted line marks the step budget of
the off-policy methods.}
\label{fig:curves-log}
\end{figure}

\begin{table}[!htbp]
\caption{RESPO at RAFALE's step budget (mean \(\pm\) SD over five seeds), read
at that step, interpolating linearly between the two bracketing checkpoints
where it was not evaluated.}
\label{tab:respo-matched}
\centering
\scriptsize
\renewcommand{\arraystretch}{0.95}
\setlength{\tabcolsep}{4pt}
\begin{tabular}{lrrrr}
\toprule
Task & Step & RESPO reward & RESPO cost & RESPO final reward \\
\midrule
Swimmer     & 1M & \(17\pm9\)   & \(16.6\pm9.8\) & 36 \\
HalfCheetah & 1M & \(434\pm231\) & \(1.9\pm2.1\) & 2344 \\
Hopper      & 2M & \(776\pm292\) & \(0.5\pm0.7\) & 1073 \\
Walker2d    & 1M & \(373\pm86\)  & \(1.2\pm0.6\) & 2925 \\
Ant         & 3M & \(117\pm41\)  & \(0.7\pm0.3\) & 1950 \\
Humanoid    & 3M & \(916\pm326\) & \(0.3\pm0.3\) & 5380 \\
CarGoal     & 1M & \(3.3\pm1.1\)     & \(38.8\pm16.8\) & 20.0 \\
\bottomrule
\end{tabular}
\end{table}

\paragraph{Lagrangian-variant curves.}
\Cref{fig:rq2-curves} shows the learning curves of the three variants of
\cref{eq:penalty-variants}.  The augmented Lagrangian lowers the peak
evaluated checkpoint cost on all three tasks; on Humanoid it attenuates the largest
excursions without uniformly lowering the trajectory and finishes with all five
seeds feasible.  

\begin{figure}[!htbp]
\centering
\includegraphics[width=\textwidth]{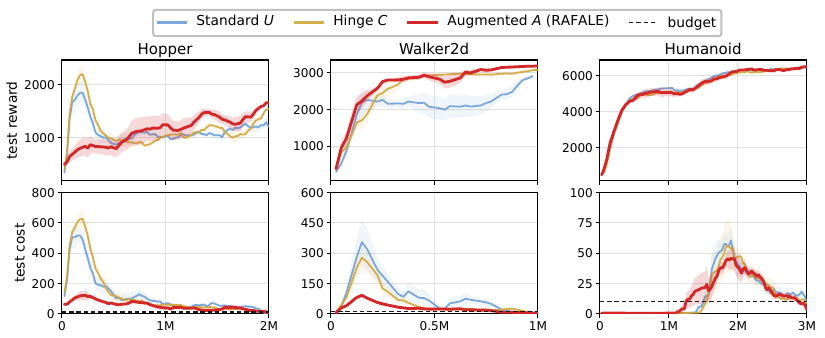}
\caption{Learning curves of the three Lagrangian variants: mean \(\pm1\)
standard error over five seeds at the evaluation steps shared by all seeds;
dashed lines mark the budget.  The augmented Lagrangian
\(L_A\) reduces transient cost excursions relative to the hinge control \(L_C\)
and the standard Lagrangian \(L\), without a systematic loss of final reward.}
\label{fig:rq2-curves}
\end{figure}

\paragraph{Actor-family curves.}
\Cref{fig:rq3-curves} shows the learning curves behind \cref{tab:actor-family}:
the flow actor against tanh-Gaussian and eight-component tanh-GMM actors under
the shared safety machinery, each with its family-native regularizer.

\begin{figure}[!htbp]
\centering
\includegraphics[width=\textwidth]{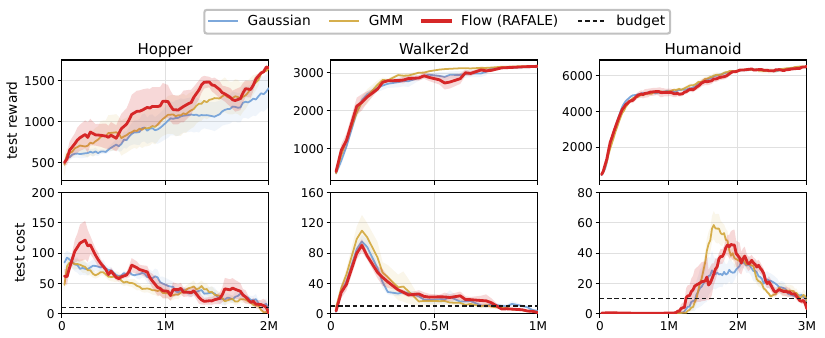}
\caption{Learning curves of the flow, tanh-Gaussian, and eight-component
tanh-GMM actors: mean \(\pm1\) standard error over five seeds at the
evaluation steps shared by all seeds, with 50 evaluation episodes per
checkpoint; dashed lines mark the budget.}
\label{fig:rq3-curves}
\end{figure}

\paragraph{Frozen-critic multiplier response.}
\Cref{fig:lambda-audit} trains the flow actor against the frozen critics of a
1M-step Hopper checkpoint for several fixed multipliers; larger multipliers shift the
generated-action cost-critic distribution left, a diagnostic of the implemented
update.

\begin{figure}[!htbp]
\centering
\includegraphics[width=\textwidth]{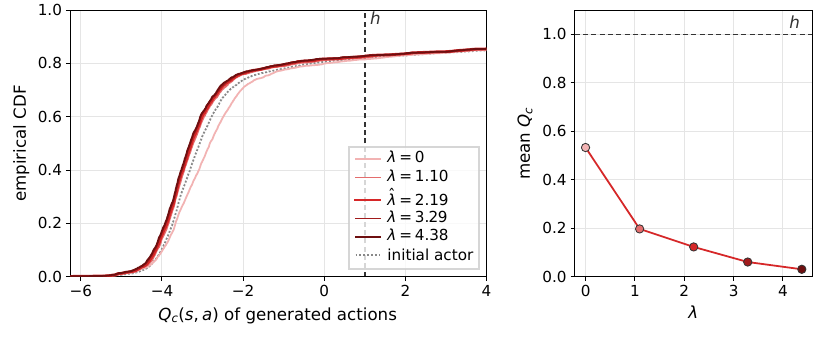}
\caption{Frozen-critic multiplier response.  For each fixed \(\lambda\), a flow
actor is trained from the same random initialization (dotted) for the same
number of updates.  Larger multipliers shift the generated-action cost-critic
CDF left.  The dashed lines mark the budget \(h\) in cost-critic units, and \(\hat\lambda\) is the multiplier stored
in the checkpoint.}
\label{fig:lambda-audit}
\end{figure}

\end{document}